\documentclass[11pt]{article}

\usepackage{amsmath,amsbsy,amsfonts,amssymb,amsthm,dsfont,fullpage,units}
\usepackage{authblk}
\usepackage[round]{natbib}
\usepackage{tikz}
\usetikzlibrary{calc,shapes}
\def\E{\mathbb{E}}
\def\R{\mathbb{R}}

\def\sphere{\mathcal{S}}
\def\I{\mathds{1}}
\def\F{{\operatorname{F}}}

\renewcommand\v[1]{\Vec{#1}}
\newcommand\dotp[1]{\langle #1 \rangle}
\def\t{{\scriptscriptstyle\top}}
\def\h{\hat}
\def\tl{\tilde}
\def\wh{\widehat}

\def\e{\v{e}}
\def\x{\v{x}}
\def\y{\v{y}}

\def\Sig{\varSigma}
\def\eps{\epsilon}
\def\veps{\varepsilon}

\DeclareMathOperator{\blkdiag}{blkdiag}
\DeclareMathOperator{\cov}{cov}

\DeclareMathOperator{\diag}{diag}

\DeclareMathOperator{\range}{range}

\DeclareMathOperator{\sign}{sign}

\def\pairs{\mathrm{Pairs}}
\def\triples{\mathrm{Triples}}

\def\algorithmA{Algorithm A}
\def\algorithmB{Algorithm B}

\newtheorem{lemma}{Lemma}[section]
\newtheorem{proposition}{Proposition}[section]
\newtheorem{theorem}{Theorem}[section]
\theoremstyle{definition}
\newtheorem{condition}{Condition}[section]

\theoremstyle{remark}

\def\bs{\kern-2pt}
\def\bss{\kern-3pt}
\def\coloneq{\kern-3pt :=}

\title{A \mbox{Method of Moments} for \mbox{Mixture Models} and
\mbox{Hidden Markov Models}}
\author[1]{Animashree Anandkumar}
\author[2]{Daniel Hsu}
\author[2]{Sham M.~Kakade}
\affil[1]{Department of EECS, University of California, Irvine}
\affil[2]{Microsoft Research New England}

\begin{document}
\maketitle
{\def\thefootnote{}
\footnotetext{E-mail:
\texttt{a.anandkumar@uci.edu},
\texttt{dahsu@microsoft.com},
\texttt{skakade@microsoft.com}}}

\begin{abstract}
Mixture models are a fundamental tool in applied statistics and machine
learning for treating data taken from multiple subpopulations.
The current practice for estimating the parameters of such models relies on
local search heuristics (\emph{e.g.}, the EM algorithm) which are prone to
failure, and existing consistent methods are unfavorable due to their high
computational and sample complexity which typically scale exponentially
with the number of mixture components.
This work develops an efficient \emph{method of moments} approach to
parameter estimation for a broad class of high-dimensional mixture models
with many components, including multi-view mixtures of Gaussians (such as
mixtures of axis-aligned Gaussians) and hidden Markov models.
The new method leads to rigorous unsupervised learning results for mixture
models that were not achieved by previous works; and, because of its
simplicity, it offers a viable alternative to EM for practical deployment.
\end{abstract}

\section{Introduction}

Mixture models are a fundamental tool in applied statistics and machine
learning for treating data taken from multiple
subpopulations~\citep{TSM85}.
In a mixture model, the data are generated from a number of possible
sources, and it is of interest to identify the nature of the individual
sources.
As such, estimating the unknown parameters of the mixture model from
sampled data---especially the parameters of the underlying constituent
distributions---is an important statistical task.
For most mixture models, including the widely used mixtures of Gaussians
and hidden Markov models (HMMs), the current practice relies on the
Expectation-Maximization (EM) algorithm, a local search heuristic for
maximum likelihood estimation.
However, EM has a number of well-documented drawbacks regularly faced by
practitioners, including slow convergence and suboptimal local
optima~\citep{RW84}.

An alternative to maximum likelihood and EM, especially in the context of
mixture models, is the \emph{method of moments} approach.
The method of moments dates back to the origins of mixture models with
Pearson's solution for identifying the parameters of a mixture of two
univariate Gaussians~\citep{Pearson94}.
In this approach, model parameters are chosen to specify a distribution
whose $p$-th order moments, for several values of $p$, are equal to the
corresponding empirical moments observed in the data.
Since Pearson's work, the method of moments has been studied and adapted
for a variety of problems; their intuitive appeal is also complemented with
a guarantee of statistical consistency under mild conditions.
Unfortunately, the method often runs into trouble with large mixtures of
high-dimensional distributions.
This is because the equations determining the parameters are typically
based on moments of order equal to the number of model parameters, and
high-order moments are exceedingly difficult to estimate accurately due to
their large variance.

This work develops a computationally efficient method of moments based on
only low-order moments that can be used to estimate the parameters of a
broad class of high-dimensional mixture models with many components.
The resulting estimators can be implemented with standard numerical linear
algebra routines (singular value and eigenvalue decompositions), and the
estimates have low variance because they only involve low-order moments.
The class of models covered by the method includes certain multivariate
Gaussian mixture models and HMMs, as well as mixture models with no
explicit likelihood equations.
The method exploits the availability of multiple indirect ``views'' of a
model's underlying latent variable that determines the source distribution,
although the notion of a ``view'' is rather general.
For instance, in an HMM, the past, present, and future observations can be
thought of as different noisy views of the present hidden state; in a
mixture of product distributions (such as axis-aligned Gaussians), the
coordinates in the output space can be partitioned (say, randomly) into
multiple non-redundant ``views''.
The new method of moments leads to unsupervised learning guarantees for
mixture models under mild rank conditions that were not achieved by
previous works; in particular, the sample complexity of accurate parameter
estimation is shown to be polynomial in the number of mixture components
and other relevant quantities.
Finally, due to its simplicity, the new method (or variants thereof) also
offers a viable alternative to EM and maximum likelihood for practical
deployment.

\subsection{Related work}

\noindent {\bf Gaussian mixture models}.
The statistical literature on mixture models is vast (a more thorough
treatment can be found in the texts of~\citet{TSM85}
and~\citet{Lindsay95}), and many advances have been made in computer
science and machine learning over the past decade or so, in part due to
their importance in modern applications.
The use of mixture models for clustering data comprises a large part of
this work, beginning with the work of~\citet{Das99} on learning
mixtures of $k$ well-separated $d$-dimensional Gaussians.
This and subsequent work~\citep{AK01,DS07,VW02,KSV05,AM05,CR08,BV08,CKLS09}
have focused on efficient algorithms that provably recover the parameters
of the constituent Gaussians from data generated by such a mixture
distribution, provided that the distance between each pair of means is
sufficiently large (roughly either $d^c$ or $k^c$ times the standard
deviation of the Gaussians, for some $c > 0$).
Such separation conditions are natural to expect in many clustering
applications, and a number of spectral projection techniques have been
shown to enhance the separation~\citep{VW02,KSV05,BV08,CKLS09}.
More recently, techniques have been developed for learning mixtures of
Gaussians without any separation condition~\citep{KMV10,BS10,MV10},
although the computational and sample complexities of these methods grow
exponentially with the number of mixture components $k$.
This dependence has also been shown to be inevitable without further
assumptions~\citep{MV10}.

\smallskip
\noindent {\bf Method of moments}.
The latter works of~\citet{BS10},~\citet{KMV10}, and~\citet{MV10} (as well
as the algorithms of~\citet{FOS05,FOS06} for a related but different
learning objective) can be thought of as modern implementations of the
method of moments, and their exponential dependence on $k$ is not
surprising given the literature on other moment methods for mixture models.
In particular, a number of moment methods for both discrete and continuous
mixture models have been developed using techniques such as the Vandermonde
decompositions of Hankel matrices~\citep{Lindsay89,LB93,BLV97,GLPR12}.
In these methods, following the spirit of Pearson's original solution, the
model parameters are derived from the roots of polynomials whose
coefficients are based on moments up to the $\Omega(k)$-th order.
The accurate estimation of such moments generally has computational and
sample complexity exponential in $k$.

\smallskip
\noindent {\bf Spectral approach to parameter estimation with low-order
moments}.
The present work is based on a notable exception to the above situation,
namely Chang's spectral decomposition technique for discrete Markov models
of evolution~\citep{Chang96} (see also~\citet{MR06} and~\citet{HKZ09} for
adaptations to other discrete mixture models such as discrete HMMs).
This spectral technique depends only on moments up to the third-order;
consequently, the resulting algorithms have computational and sample
complexity that scales only polynomially in the number of mixture
components $k$.
The success of the technique depends on a certain rank condition of the
transition matrices; but this condition is much milder than separation
conditions of clustering works, and it remains sufficient even when the
dimension of the observation space is very large~\citep{HKZ09}.
In this work, we extend Chang's spectral technique to develop a general
method of moments approach to parameter estimation, which is applicable to
a large class of mixture models and HMMs with both discrete and continuous
component distributions in high-dimensional spaces.
Like the moment methods of~\citet{MV10} and~\citet{BS10}, our algorithm
does not require a separation condition; but unlike those previous methods,
the algorithm has computational and sample complexity polynomial in $k$.

Some previous spectral approaches for related learning problems only use
second-order moments, but these approaches can only estimate a subspace
containing the parameter vectors and not the parameters
themselves~\citep{McSherry01}.
Indeed, it is known that the parameters of even very simple discrete
mixture models are not generally identifiable from only second-order
moments~\citep{Chang96}\footnote{See Appendix~\ref{appendix:nonident} for
an example of~\citet{Chang96} demonstrating the non-identifiability of
parameters from only second-order moments in a simple class of Markov
models.}.
We note that moments beyond the second-order (specifically, fourth-order
moments) have been exploited in the methods of~\citet{FJK96}
and~\citet{NR09} for the problem of learning a parallelepiped from random
samples, and that these methods are very related to techniques used for
\emph{independent component analysis}~\citep{HO00}.
Adapting these techniques for other parameter estimation problems is an
enticing possibility.

\smallskip
\noindent {\bf Multi-view learning}.
The spectral technique we employ depends on the availability of multiple
views, and such a multi-view assumption has been exploited in previous
works on learning mixtures of well-separated
distributions~\citep{CR08,CKLS09}.
In these previous works, a projection based on a \emph{canonical
correlation analysis}~\citep{Hotelling35} between two views is used to
reinforce the separation between the mixture components, and to cancel out
noise orthogonal to the separation directions.
The present work, which uses similar correlation-based projections, shows
that the availability of a third view of the data can remove the separation
condition entirely.
The multi-view assumption substantially generalizes the case where the
component distributions are product distributions (such as axis-aligned
Gaussians), which has been previously studied in the
literature~\citep{Das99,VW02,CR08,FOS05,FOS06}; the combination of this and
a non-degeneracy assumption is what allows us to avoid the sample
complexity lower bound of~\citet{MV10} for Gaussian mixture models.
The multi-view assumption also naturally arises in many applications, such
as in multimedia data with (say) text, audio, and video
components~\citep{BL08,CKLS09}; as well as in linguistic data, where the
different words in a sentence or paragraph are considered noisy predictors
of the underlying semantics~\citep{GCY92}.
In the vein of this latter example, we consider estimation in a simple
bag-of-words document topic model as a warm-up to our general method; even
this simpler model illustrates the power of pair-wise and triple-wise
(\emph{i.e.}, bigram and trigram) statistics that were not exploited by
previous works on multi-view learning.

\subsection{Outline}
Section~\ref{section:topic} first develops the method of moments in the
context of a simple discrete mixture model motivated by document topic
modeling; an explicit algorithm and convergence analysis are also provided.
The general setting is considered in Section~\ref{section:general}, where
the main algorithm and its accompanying correctness and efficiency
guarantee are presented. 
Applications to learning multi-view mixtures of Gaussians and HMMs are
discussed in Section~\ref{section:applications}.
All proofs are given in the appendix.

\subsection{Notations}

The standard inner product between vectors $\v{u}$ and $\v{v}$ is denoted
by $\dotp{\v{u},\v{v}} = \v{u}^\t \v{v}$.
We denote the $p$-norm of a vector $\v{v}$ by $\|\v{v}\|_p$.
For a matrix $A \in \R^{m \times n}$, we let $\|A\|_2$ denote its spectral
norm $\|A\|_2 := \sup_{\v{v} \neq \v0} \|A\v{v}\|_2 / \|\v{v}\|_2$,
$\|A\|_\F$ denote its Frobenius norm, $\sigma_i(A)$ denote the $i$-th
largest singular value, and $\kappa(A) := \sigma_1(A) /
\sigma_{\min(m,n)}(A)$ denote its condition number.
Let $\Delta^{n-1} := \{ (p_1,p_2,\dotsc,p_n) \in \R^n : p_i \geq 0 \
\forall i, \ \sum_{i=1}^n p_i = 1 \}$ denote the probability simplex in
$\R^n$, and let $\sphere^{n-1} := \{ \v{u} \in \R^n : \|\v{u}\|_2 = 1 \}$
denote the unit sphere in $\R^n$.
Let $\e_i \in \R^d$ denote the $i$-th coordinate vector whose $i$-th entry
is $1$ and the rest are zero.
Finally, for a positive integer $n$, let $[n] := \{1,2,\dotsc,n\}$.

\section{Warm-up: bag-of-words document topic modeling} \label{section:topic}

We first describe our method of moments in the simpler context of
bag-of-words models for documents.

\subsection{Setting} \label{section:topic-setting}

Suppose a document corpus can be partitioned by topic, with each document
being assigned a single topic.
Further, suppose the words in a document are drawn independently from a
multinomial distribution corresponding to the document's topic.
Let $k$ be the number of distinct topics in the corpus, $d$ be the number
of distinct words in the vocabulary, and $\ell \geq 3$ be the number of
words in each document (so the documents may be quite short).

The generative process for a document is given as follows:
\begin{enumerate}
\item The document's topic is drawn according to the multinomial
distribution specified by the probability vector $\v{w} =
(w_1,w_2,\dotsc,w_k) \in \Delta^{k-1}$.
This is modeled as a discrete random variable $h$ such that
\[ \Pr[h = j] = w_j , \quad j \in [k] . \]

\item Given the topic $h$, the document's $\ell$ words are drawn
independently according to the multinomial distribution specified by the
probability vector $\v\mu_h \in \Delta^{d-1}$.
The random vectors $\x_1, \x_2, \dotsc, \x_\ell \in \R^d$
represent the $\ell$ words by setting
\[ \x_v = \e_i \ \Leftrightarrow \ \text{the $v$-th word
in the document is $i$} , \quad i \in [d] \]
(the reason for this encoding of words will become clear in the next
section).
Therefore, for each word $v \in [\ell]$ in the document,
\[ \Pr[\x_v = \e_i | h = j] = \dotp{\e_i,\v\mu_j} = M_{i,j} , \quad i \in
[d] , j \in [k] , \]
where $M \in \R^{d \times k}$ is the matrix of conditional probabilities
$M := [\v\mu_1 | \v\mu_2 | \dotsb | \v\mu_k]$.

\end{enumerate}
This probabilistic model has the conditional independence structure
depicted in Figure~\ref{fig:graphical-model}(a) as a directed graphical
model.

We assume the following condition on $\v{w}$ and $M$.
\begin{condition}[Non-degeneracy: document topic model] \label{cond:topic}
$w_j\bs >\bs 0$ for all $j\bs \in\bs [k]$, and $M$ has rank $k$.
\end{condition}
This condition requires that each topic has non-zero probability, and also
prevents any topic's word distribution from being a mixture of the other
topics' word distributions.

\subsection{Pair-wise and triple-wise probabilities}
\label{section:topic-moments}

Define $\pairs \in \R^{d \times d}$ to be the matrix of pair-wise
probabilities whose $(i,j)$-th entry is
\[ \pairs_{i,j} := \Pr[ \x_1 = \e_i, \x_2 = \e_j ] , \quad i,j \in [d] . \]
Also define $\triples \in \R^{d \times d \times d}$ to be the third-order
tensor of triple-wise probabilities whose $(i,j,\kappa)$-th entry is
\[ \triples_{i,j,\kappa} := \Pr[ \x_1 = \e_i, \x_2 = \e_j, \x_3 = \e_\kappa
]
, \quad i,j,\kappa \in [d] .
\]
The identification of words with coordinate vectors allows $\pairs$ and
$\triples$ to be viewed as expectations of tensor products of the random
vectors $\x_1$, $\x_2$, and $\x_3$:
\begin{equation} \label{eq:pairs-triples}
\pairs = \E[ \x_1 \otimes \x_2 ]
\quad\text{and}\quad
\triples = \E[ \x_1 \otimes \x_2 \otimes \x_3 ]
.
\end{equation}
We may also view $\triples$ as a linear operator $\triples \colon \R^d \to
\R^{d \times d}$ given by
\begin{equation*} \label{eq:triples-eta}
\triples(\v\eta) := \E[ (\x_1 \otimes \x_2) \dotp{\v\eta,\x_3} ]
.
\end{equation*}
In other words, the $(i,j)$-th entry of $\triples(\v\eta)$ for $\v\eta =
(\eta_1,\eta_2,\dotsc,\eta_d)$ is
\[ \triples(\v\eta)_{i,j}
= \sum_{x=1}^d \eta_x \triples_{i,j,x}
= \sum_{x=1}^d \eta_x \triples(\e_x)_{i,j}
.
\]

The following lemma shows that $\pairs$ and $\triples(\v\eta)$ can be
viewed as certain matrix products involving the model parameters $M$ and
$\v{w}$.
\begin{lemma} \label{lemma:topic-moments}
$\pairs\bs =\bs M \diag(\v{w}) M^\t$
and
$\triples(\v\eta)\bs =\bs M \diag(M^\t \v\eta) \diag(\v{w}) M^\t$
for all $\v\eta \in \R^d$.
\end{lemma}
\begin{proof}
Since $\x_1$, $\x_2$, and $\x_3$ are conditionally independent given $h$,
\begin{align*}
\pairs_{i,j}
& = \Pr[ \x_1 = \e_i, \x_2 = \e_j ]
= \sum_{t=1}^k \Pr[ \x_1 = \e_i, \x_2 = \e_j | h=t ] \cdot \Pr[h=t] \\
& = \sum_{t=1}^k \Pr[ \x_1 = \e_i | h=t]
\cdot \Pr[ \x_2 = \e_j | h=t]
\cdot \Pr[h=t]
= \sum_{t=1}^k M_{i,t} \cdot M_{j,t} \cdot w_t
\end{align*}
so $\pairs = M \diag(\v{w}) M^\t$.
Moreover, writing $\v\eta = (\eta_1,\eta_2,\dotsc,\eta_d)$,
\begin{align*}
\triples(\v\eta)_{i,j}
& = \sum_{x=1}^d \eta_x \Pr[\x_1 = \e_i, \x_2 = \e_j, \x_3 = \e_x]
\\
& = \sum_{x=1}^d \sum_{t=1}^k \eta_x \cdot M_{i,t} \cdot M_{j,t} \cdot
M_{x,t} \cdot w_t
= \sum_{t=1}^k M_{i,t} \cdot M_{j,t} \cdot w_t \cdot (M^\t \v\eta)_t
\end{align*}
so $\triples(\v\eta) = M \diag(M^\t \v\eta) \diag(\v{w}) M^\t$.
\end{proof}

\subsection{Observable operators and their spectral properties}
\label{section:topic-operator}

The pair-wise and triple-wise probabilities can be related in a way that
essentially reveals the conditional probability matrix $M$.
This is achieved through a matrix called an ``observable operator''.
Similar observable operators were previously used to characterize
multiplicity automata~\citep{Sch61,Jaeger00} and, more recently, for
learning discrete HMMs (via an operator parameterization)~\citep{HKZ09}.

\begin{lemma} \label{lemma:topic-operator}
Assume Condition~\ref{cond:topic}.
Let $U \in \R^{d \times k}$ and $V \in \R^{d \times k}$ be matrices such
that both $U^\t M$ and $V^\t M$ are invertible.
Then $U^\t \pairs V$ is invertible, and for all $\v\eta \in \R^d$, the
``observable operator'' $B(\v\eta) \in \R^{k \times k}$, given by
\begin{equation*} 
B(\v\eta) := (U^\t \triples(\v\eta) V) (U^\t \pairs V)^{-1}
,
\end{equation*}
satisfies
\[ B(\v\eta) = (U^\t M) \diag(M^\t \v\eta) (U^\t M)^{-1} . \]
\end{lemma}
\begin{proof}
Since $\diag(\v{w}) \succ 0$ by Condition~\ref{cond:topic} and $U^\t \pairs
V = (U^\t M) \diag(\v{w}) M^\t V$ by Lemma~\ref{lemma:topic-moments}, it
follows that $U^\t \pairs V$ is invertible by the assumptions on $U$ and
$V$.
Moreover, also by Lemma~\ref{lemma:topic-moments},
\begin{align*}
B(\v\eta)
& = (U^\t \triples(\v\eta) V) \ (U^\t \pairs V)^{-1}
\\
& = (U^\t M \diag(M^\t \v\eta) \diag(\v{w}) M^\t V) \
(U^\t \pairs V)^{-1}
\\
& = (U^\t M) \diag(M^\t \v\eta) (U^\t M)^{-1} \
(U^\t M \diag(\v{w}) M^\t V) \
(U^\t \pairs V)^{-1}
\\
& = (U^\t M) \diag(M^\t \v\eta) (U^\t M)^{-1}
.
\qedhere
\end{align*}
\end{proof}
The matrix $B(\v\eta)$ is called ``observable'' because it is only a
function of the observable variables' joint probabilities (\emph{e.g.},
$\Pr[\x_1 = \e_i, \x_2 = \e_j]$).
In the case $\v\eta = \e_x$ for some $x \in [d]$, the matrix $B(\e_x)$ is
\emph{similar} (in the linear algebraic sense) to the diagonal matrix
$\diag(M^\t \e_x)$; the collection of matrices $\{ \diag(M^\t \e_x) : x \in
[d] \}$ (together with $\v{w}$) can be used to compute joint probabilities
under the model (see, \emph{e.g.},~\citet{HKZ09}).
Note that the columns of $U^\t M$ are eigenvectors of $B(\e_x)$, with the
$j$-th column having an associated eigenvalue equal to $\Pr[\x_v = x |
h=j]$.
If the word $x$ has distinct probabilities under every topic, then
$B(\e_x)$ has exactly $k$ distinct eigenvalues, each having geometric
multiplicity one and corresponding to a column of $U^\t M$.

\subsection{Topic-word distribution estimator and convergence guarantee}
\label{section:topic-method}

The spectral properties of the observable operators $B(\v\eta)$ implied by
Lemma~\ref{lemma:topic-operator} suggest the estimation procedure
(\algorithmA) in Figure~\ref{fig:topic}.
The procedure is essentially a plug-in approach based on the equations
relating the second- and third-order moments in
Lemma~\ref{lemma:topic-operator}.
We focus on estimating $M$; estimating the mixing weights $\v{w}$ is easily
handled as a secondary step (see Appendix~\ref{appendix:mixing-weights} for
the estimator in the context of the general model in
Section~\ref{section:general-setting}).

\begin{figure}
\framebox[\textwidth]{\small\begin{minipage}{0.95\textwidth}
\textbf{\algorithmA}
\begin{enumerate}

\item Obtain empirical frequencies of word pairs and triples from a given
sample of documents, and form the tables $\wh\pairs \in \R^{d \times d}$
and $\wh\triples \in \R^{d \times d \times d}$ corresponding to the
population quantities $\pairs$ and $\triples$.

\item Let $\h{U} \in \R^{d \times k}$ and $\h{V} \in \R^{d \times k}$ be,
respectively, matrices of orthonormal left and right singular vectors of
$\wh\pairs$ corresponding to its top $k$ singular values.

\item Pick $\v\eta \in \R^d$ (see remark in the main text), and compute the
right eigenvectors $\h\xi_1, \h\xi_2, \dotsc, \h\xi_k$ (of unit Euclidean
norm) of
\[ \h{B}(\v\eta) := (\h{U}^\t \wh\triples(\v\eta) \h{V}) (\h{U}^\t
\wh\pairs \h{V})^{-1} . \]
(Fail if not possible.)

\item Let $\h\mu_j := \h{U} \h\xi_j / \dotp{\v1,\h{U} \h\xi_j}$ for all $j
\in [k]$.

\item Return $\h{M} := [ \h\mu_1 | \h\mu_2 | \dotsb | \h\mu_k ]$.

\end{enumerate}
\end{minipage}}
\vspace{-4mm}
\caption{Topic-word distribution estimator (\algorithmA).}
\label{fig:topic}
\end{figure}

\smallskip
\noindent {\bf On the choice of $\v\eta$}.
As discussed in the previous section, a suitable choice for $\v\eta$ can be
based on prior knowledge about the topic-word distributions, such as
$\v\eta = \e_x$ for some $x \in [d]$ that has different
conditional probabilities under each topic.
In the absence of such information, one may select $\v\eta$ randomly from
the subspace $\range(\h{U})$.
Specifically, take $\v\eta := \h{U}\v\theta$ where $\v\theta \in \R^k$ is a
random unit vector distributed uniformly over $\sphere^{k-1}$.

\smallskip
The following theorem establishes the convergence rate of \algorithmA.

\begin{theorem} \label{theorem:topic}
There exists a constant $C>0$ such that the following holds.
Pick any $\delta \in (0,1)$.
Assume the document topic model from Section~\ref{section:topic-setting}
satisfies Condition~\ref{cond:topic}.
Further, assume that in \algorithmA, $\wh\pairs$ and $\wh\triples$ are,
respectively, the empirical averages of $N$ independent copies of $\x_1
\otimes \x_2$ and $\x_1 \otimes \x_2 \otimes \x_3$; and that $\v\eta =
\h{U}\v\theta$ where $\v\theta \in \R^k$ is an independent random unit
vector distributed uniformly over $\sphere^{k-1}$.
If
\[ N \geq C \cdot \frac{k^7 \cdot \ln(1/\delta)}{\sigma_k(M)^6 \cdot
\sigma_k(\pairs)^4 \cdot \delta^2}
,
\]
then with probability at least $1-\delta$, the parameters returned by
\algorithmA\ have the following guarantee:
there exists a permutation $\tau$ on $[k]$ and scalars $c_1, c_2, \dotsc,
c_k \in \R$ such that, for each $j \in [k]$,
\[ \|c_j \h\mu_j - \v\mu_{\tau(j)}\|_2 \leq
C \cdot \|\v\mu_{\tau(j)}\|_2 \cdot \frac{k^5}{\sigma_k(M)^4 \cdot
\sigma_k(\pairs)^2 \cdot \delta} \cdot \sqrt{\frac{\ln(1/\delta)}{N}}
.
\]
\end{theorem}

The proof of Theorem~\ref{theorem:topic}, as well as some illustrative
empirical results on using \algorithmA, are presented in
Appendix~\ref{appendix:topic}.
A few remarks about the theorem are in order.

\smallskip
\noindent {\bf On boosting the confidence}.
Although the convergence depends polynomially on $1/\delta$, where $\delta$
is the failure probability, it is possible to boost the confidence by
repeating Step 3 of \algorithmA\ with different random $\v\eta$ until the
eigenvalues of $\h{B}(\v\eta)$ are sufficiently separated (as judged by
confidence intervals).

\smallskip
\noindent {\bf On the scaling factors $c_j$}.
With a larger sample complexity that depends on $d$, an error bound can be
established for $\|\h\mu_j - \v\mu_{\tau(j)}\|_1$ directly (without the
unknown scaling factors $c_j$).
We also remark that the scaling factors can be estimated from the
eigenvalues of $\h{B}(\v\eta)$, but we do not pursue this approach as it is
subsumed by \algorithmB\ anyway.

\section{A method of moments for multi-view mixture models}
\label{section:general}

We now consider a much broader class of mixture models and present a
general method of moments in this context.

\subsection{General setting}
\label{section:general-setting}

Consider the following multi-view mixture model; $k$ denotes the number of
mixture components, and $\ell$ denotes the number of views.
We assume $\ell \geq 3$ throughout.
Let $\v{w} = (w_1,w_2,\dotsc,w_k) \in \Delta^{k-1}$ be a vector of mixing
weights, and let $h$ be a (hidden) discrete random variable with $\Pr[h=j]
= w_j$ for all $j \in [k]$.
Let $\x_1,\x_2,\dotsc,\x_\ell \in \R^d$ be $\ell$ random vectors that are
conditionally independent given $h$;
the directed graphical model is depicted in
Figure~\ref{fig:graphical-model}(a).

Define the conditional mean vectors as
\[ \v\mu_{v,j} := \E[\x_v | h = j] , \quad v \in [\ell], j \in [k]
, \]
and let $M_v \in \R^{d \times k}$ be the matrix whose $j$-th column is
$\v\mu_{v,j}$.
Note that we do not specify anything else about the (conditional)
distribution of $\x_v$---it may be continuous, discrete, or even a hybrid
depending on $h$.

\begin{figure}
\begin{center}
\begin{tabular}{cc}
\begin{tikzpicture}
  [
    scale=1.0,
    observed/.style={circle,minimum size=0.7cm,inner sep=0mm,draw=black,fill=black!20},
    hidden/.style={circle,minimum size=0.7cm,inner sep=0mm,draw=black},
  ]
  \node [hidden,name=h] at ($(0,0)$) {$h$};
  \node [observed,name=x1] at ($(-1.5,-1)$) {$\x_1$};
  \node [observed,name=x2] at ($(-0.5,-1)$) {$\x_2$};
  \node at ($(0.5,-1)$) {$\dotsb$};
  \node [observed,name=xl] at ($(1.5,-1)$) {$\x_\ell$};
  \draw [->] (h) to (x1);
  \draw [->] (h) to (x2);
  \draw [->] (h) to (xl);
\end{tikzpicture}
\qquad
&
\qquad
\begin{tikzpicture}
  [
    scale=1.0,
    observed/.style={circle,minimum size=0.7cm,inner sep=0mm,draw=black,fill=black!20},
    hidden/.style={circle,minimum size=0.7cm,inner sep=0mm,draw=black},
  ]
  \node [hidden,name=h1] at ($(-1.2,0)$) {$h_1$};
  \node [hidden,name=h2] at ($(0,0)$) {$h_2$};
  \node [name=hd] at ($(1.2,0)$) {$\dotsb$};
  \node [hidden,name=hl] at ($(2.4,0)$) {$h_\ell$};
  \node [observed,name=x1] at ($(-1.2,-1)$) {$\x_1$};
  \node [observed,name=x2] at ($(0,-1)$) {$\x_2$};
  \node [observed,name=xl] at ($(2.4,-1)$) {$\x_\ell$};
  \draw [->] (h1) to (h2);
  \draw [->] (h2) to (hd);
  \draw [->] (hd) to (hl);
  \draw [->] (h1) to (x1);
  \draw [->] (h2) to (x2);
  \draw [->] (hl) to (xl);
\end{tikzpicture}
\\
(a) & (b)
\end{tabular}
\vspace{-5mm}
\end{center}
\caption{(a) The multi-view mixture model.
(b) A hidden Markov model.}
\label{fig:graphical-model}
\vspace{-1mm}
\end{figure}
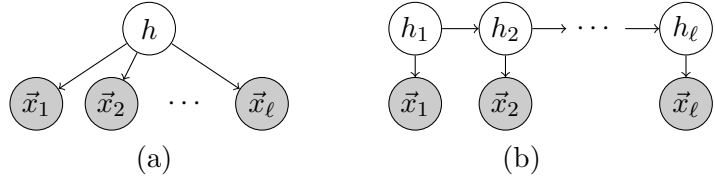

We assume the following conditions on $\v{w}$ and the $M_v$.
\begin{condition}[Non-degeneracy: general setting] \label{cond:general}
$w_j > 0$ for all $j \in [k]$, and $M_v$ has rank $k$ for all
$v \in [\ell]$.
\end{condition}
We remark that it is easy to generalize to the case where views have
different dimensionality (\emph{e.g.}, $\x_v \in \R^{d_v}$ for possibly
different dimensions $d_v$).
For notational simplicity, we stick to the same dimension for each view.
Moreover, Condition~\ref{cond:general} can be relaxed in some cases; we
discuss one such case in Section~\ref{section:gmm} in the context of
Gaussian mixture models.

Because the conditional distribution of $\x_v$ is not specified beyond its
conditional means, it is not possible to develop a maximum likelihood
approach to parameter estimation.
Instead, as in the document topic model, we develop a method of moments
based on solving polynomial equations arising from eigenvalue problems.

\subsection{Observable moments and operators}
\label{section:general-moments}

We focus on the moments concerning $\{\x_1, \x_2, \x_3\}$, but the same
properties hold for other triples of the random vectors $\{\x_a, \x_b,
\x_c\} \subseteq \{\x_v : v \in [\ell] \}$ as well.

As in~\eqref{eq:pairs-triples}, we define the matrix $P_{1,2} \in \R^{d
\times d}$ of second-order moments, and the tensor $P_{1,2,3} \in \R^{d
\times d \times d}$ of third-order moments, by
\[ P_{1,2} := \E[\x_1 \otimes \x_2]
\quad\text{and}\quad
P_{1,2,3} := \E[\x_1 \otimes \x_2 \otimes \x_3] . \]
Again, $P_{1,2,3}$ is regarded as the linear operator $P_{1,2,3} \colon
\v\eta \mapsto \E[ (\x_1 \otimes \x_2) \dotp{\v\eta,\x_3}]$.

Lemma~\ref{lemma:general-moments} and Lemma~\ref{lemma:general-operator}
are straightforward generalizations of Lemma~\ref{lemma:topic-moments} and
Lemma~\ref{lemma:topic-operator}.
\begin{lemma} \label{lemma:general-moments}
$P_{1,2}\bs =\bs M_1 \diag(\v{w}) M_2^\t$
and
$P_{1,2,3}(\v\eta)\bs =\bs M_1 \diag(M_3^\t \v\eta) \diag(\v{w}) M_2^\t$
for all $\v\eta \in \R^d$.
\end{lemma}

\begin{lemma} \label{lemma:general-operator}
Assume Condition~\ref{cond:general}.
For $v \in \{1,2,3\}$, let $U_v \in \R^{d \times k}$ be a matrix such that
$U_v^\t M_v$ is invertible.
Then $U_1^\t P_{1,2} U_2$ is invertible, and for all $\v\eta \in \R^d$, the
``observable operator'' $B_{1,2,3}(\v\eta) \in \R^{k \times k}$, given by
$B_{1,2,3}(\v\eta) := (U_1^\t P_{1,2,3}(\v\eta) U_2) (U_1^\t P_{1,2}
U_2)^{-1}$, satisfies
\[ B_{1,2,3}(\v\eta) = (U_1^\t M_1) \diag(M_3^\t \v\eta) (U_1^\t M_1)^{-1}
. \]
In particular, the $k$ roots of the polynomial $\lambda \mapsto
\det(B_{1,2,3}(\v\eta) - \lambda I)$ are
$\{ \dotp{\v\eta,\v\mu_{3,j}} : j \in [k] \}$.
\end{lemma}

Recall that \algorithmA\ relates the eigenvectors of $B(\v\eta)$ to the
matrix of conditional means $M$.
However, eigenvectors are only defined up to a scaling of each vector;
without prior knowledge of the correct scaling, the eigenvectors are not
sufficient to recover the parameters $M$.
Nevertheless, the eigenvalues also carry information about the parameters,
as shown in Lemma~\ref{lemma:general-operator}, and it is possible to
reconstruct the parameters from different the observation operators applied
to different vectors $\v\eta$.
This idea is captured in the following lemma.

\begin{lemma} \label{lemma:general-system}
Consider the setting and definitions from
Lemma~\ref{lemma:general-operator}.
Let $\Theta \in \R^{k \times k}$ be an invertible matrix, and let
$\v\theta_i^\t \in \R^k$ be its $i$-th row.
Moreover, for all $i \in [k]$, let $\lambda_{i,1},
\lambda_{i,2}, \dotsc, \lambda_{i,k}$ denote the $k$ eigenvalues of
$B_{1,2,3}(U_3\v\theta_i)$ in the order specified by the matrix of right
eigenvectors $U_1^\t M_1$.
Let $L \in \R^{k \times k}$ be the matrix whose $(i,j)$-th entry is
$\lambda_{i,j}$.
Then
\begin{equation*} 
\Theta U_3^\t M_3 = L
.
\end{equation*}
\end{lemma}

Observe that the unknown parameters $M_3$ are expressed as the solution to
a linear system in the above equation, where the elements of the right-hand
side $L$ are the roots of $k$-th degree polynomials derived from the
second- and third-order observable moments (namely, the characteristic
polynomials of the $B_{1,2,3}(U_3\v\theta_i)$, $\forall i \in [k]$).
This template is also found in other moment methods based on decompositions
of a Hankel matrix.
A crucial distinction, however, is that the $k$-th degree polynomials in
Lemma~\ref{lemma:general-system} only involve low-order moments, whereas
standard methods may involve up to $\Omega(k)$-th order moments which are
difficult to estimate~\citep{Lindsay89,LB93,GLPR12}.

\subsection{Main result: general estimation procedure and sample complexity
bound}
\label{section:general-method}

The lemmas in the previous section suggest the estimation procedure
(\algorithmB) presented in Figure~\ref{fig:general}.

\begin{figure}[h]
\framebox[\textwidth]{\small\begin{minipage}{0.95\textwidth}
\textbf{\algorithmB}
\begin{enumerate}
\item Compute empirical averages from $N$ independent copies of $\x_1
\otimes \x_2$ to form $\h{P}_{1,2} \in \R^{d \times d}$.
Similarly do the same for $\x_1 \otimes \x_3$ to form $\h{P}_{1,3} \in
\R^{k \times k}$, and for $\x_1 \otimes \x_2 \otimes \x_3$ to form
$\h{P}_{1,2,3} \in \R^{d \times d \times d}$.

\item Let $\h{U}_1 \in \R^{d \times k}$ and $\h{U}_2 \in \R^{d \times k}$
be, respectively, matrices of orthonormal left and right singular vectors
of $\h{P}_{1,2}$ corresponding to its top $k$ singular values.
Let $\h{U}_3 \in \R^{d \times k}$ be the matrix of orthonormal right
singular vectors of $\h{P}_{1,3}$ corresponding to its top $k$ singular
values.

\item Pick an invertible matrix $\Theta \in \R^{k \times k}$, with its
$i$-th row denoted as $\v\theta_i^\t \in \R^k$.
In the absence of any prior information about $M_3$, a suitable choice for
$\Theta$ is a random rotation matrix.

Form the matrix \ \
$\h{B}_{1,2,3}(\h{U}_3\v\theta_1) := (\h{U}_1^\t
\h{P}_{1,2,3}(\h{U}_3\v\theta_1) \h{U}_2) (\h{U}_1^\t \h{P}_{1,2}
\h{U}_2)^{-1}$.

Compute $\h{R}_1 \in \R^{k \times k}$ (with unit Euclidean norm columns)
that diagonalizes $\h{B}_{1,2,3}(\h{U}_3\v\theta_1)$, \emph{i.e.}, \ \
$\h{R}_1^{-1} \h{B}_{1,2,3}(\h{U}_3\v\theta_1) \h{R}_1 =
\diag(\h\lambda_{1,1}, \h\lambda_{1,2}, \dotsc, \h\lambda_{1,k})$.
\ \ (Fail if not possible.)

\item For each $i \in \{2,\dotsc,k\}$, obtain the diagonal entries
$\h\lambda_{i,1}, \h\lambda_{i,2}, \dotsc, \h\lambda_{i,k}$ of
$\h{R}_1^{-1} \h{B}_{1,2,3}(\h{U}_3\v\theta_i) \h{R}_1$, and form the
matrix $\h{L} \in \R^{k \times k}$ whose $(i,j)$-th entry is
$\h\lambda_{i,j}$.

\item Return $\h{M}_3 := \h{U}_3 \Theta^{-1} \h{L}$.

\end{enumerate}
\end{minipage}}
\vspace{-4mm}
\caption{General method of moments estimator (\algorithmB).}
\label{fig:general}
\end{figure}

As stated, the \algorithmB\ yields an estimator for $M_3$, but the method
can easily be applied to estimate $M_v$ for all other views $v$.
One caveat is that the estimators may not yield the same ordering of the
columns, due to the unspecified order of the eigenvectors obtained in the
third step of the method, and therefore some care is needed to obtain a
consistent ordering.
We outline one solution in Appendix~\ref{appendix:ordering}.

The sample complexity of \algorithmB\ depends on the specific concentration
properties of $\x_1, \x_2, \x_3$.
We abstract away this dependence in the following condition.
\begin{condition} \label{cond:concentration}
There exist positive scalars $N_0$, $C_{1,2}$, $C_{1,3}$, $C_{1,2,3}$, and
a function $f(N,\delta)$ (decreasing in $N$ and $\delta$) such that for any
$N \geq N_0$ and $\delta \in (0,1)$,
\begin{enumerate}
\item $\Pr\Bigl[
\|\h{P}_{a,b} - P_{a,b}\|_2
\leq C_{a,b} \cdot f(N,\delta) 
\Bigr] \geq 1-\delta$
\quad for $\{a,b\} \in \{\{1,2\},\{1,3\}\}$,

\item $\forall \v{v} \in \R^d$, \
$\Pr\Bigl[
\|\h{P}_{1,2,3}(\v{v}) - P_{1,2,3}(\v{v})\|_2
\leq C_{1,2,3} \cdot \|\v{v}\|_2 \cdot f(N,\delta) 
\Bigr] \geq 1-\delta$.
\end{enumerate}
Moreover (for technical convenience), $\h{P}_{1,3}$ is independent of
$\h{P}_{1,2,3}$ (which may be achieved, say, by splitting a sample of size
$2N$).
\end{condition}
For the discrete models such as the document topic model of
Section~\ref{section:topic-setting} and discrete HMMs~\citep{MR06,HKZ09},
Condition~\ref{cond:concentration} holds with $N_0 = C_{1,2} = C_{1,3}
= C_{1,2,3} = 1$, and $f(N,\delta) = (1+\sqrt{\ln(1/\delta)}) / \sqrt{N}$.
Using standard techniques~(\emph{e.g.}, \citet{CKLS09,Vershynin12}), the
condition can also be shown to hold for mixtures of various continuous
distributions such as multivariate Gaussians.

Now we are ready to present the main theorem of this section (proved in
Appendix~\ref{appendix:general-proof}).

\begin{theorem} \label{theorem:general}
There exists a constant $C>0$ such that the following holds.
Assume the three-view mixture model satisfies Condition~\ref{cond:general}
and Condition~\ref{cond:concentration}.
Pick any $\epsilon \in (0,1)$ and $\delta \in (0,\delta_0)$.
Further, assume $\Theta \in \R^{k \times k}$ is an independent random
rotation matrix distributed uniformly over the Stiefel manifold $\{ Q \in
\R^{k \times k} : Q^\t Q = I \}$.
If the number of samples $N$ satisfies $N \geq N_0$ and
\begin{align*}
f(N,\delta/k) & \leq C \cdot \frac
{\min_{i \neq j}\|M_3(\e_i-\e_j)\|_2 \cdot \sigma_k(P_{1,2})}
{C_{1,2,3} \cdot k^5 \cdot \kappa(M_1)^4}
\cdot \frac
{\delta}
{\ln(k/\delta)}
\cdot \eps
,
\\
f(N,\delta) & \leq C \cdot \min\Biggl\{
\frac
{\min_{i \neq j}\|M_3(\e_i-\e_j)\|_2 \cdot \sigma_k(P_{1,2})^2}
{C_{1,2} \cdot \|P_{1,2,3}\|_2 \cdot k^5 \cdot \kappa(M_1)^4}
\cdot \frac
{\delta}
{\ln(k/\delta)}
, \
\frac
{\sigma_k(P_{1,3})}
{C_{1,3}}
\Biggr\} \cdot \eps
\end{align*}
where $\|P_{1,2,3}\|_2 := \max_{\v{v} \neq \v0} \|P_{1,2,3}(\v{v})\|_2$,
then with probability at least $1-5\delta$, \algorithmB\ returns $\h{M}_3 =
[\h\mu_{3,1}|\h\mu_{3,2}|\dotsb|\h\mu_{3,k}]$ with the following guarantee:
there exists a permutation $\tau$ on $[k]$ such that for each $j \in [k]$,
\begin{equation*}
\|\h\mu_{3,j} - \v\mu_{3,\tau(j)}\|_2
\leq \max_{j' \in [k]} \|\v\mu_{3,j'}\|_2 \cdot \epsilon
.
\end{equation*}
\end{theorem}

\section{Applications}
\label{section:applications}

In addition to the document clustering model from
Section~\ref{section:topic}, a number of natural latent variable models fit
into this multi-view framework.
We describe two such cases in this section: Gaussian mixture models and
HMMs, both of which have been (at least partially) studied in the
literature.
In both cases, the estimation technique of \algorithmB\ leads to new
learnability results that were not achieved by previous works.

\subsection{Multi-view Gaussian mixture models}
\label{section:gmm}

The standard Gaussian mixture model is parameterized by a mixing weight
$w_j$, mean vector $\v\mu_j \in \R^D$, and covariance matrix $\Sig_j \in
\R^{D \times D}$ for each mixture component $j \in [k]$.
The hidden discrete random variable $h$ selects a component $j$ with
probability $\Pr[h = j] = w_j$; the conditional distribution of the
observed random vector $\x$ given $h$ is a multivariate Gaussian with mean
$\v\mu_h$ and covariance $\Sig_h$.

The multi-view assumption for Gaussian mixture models asserts that for each
component $j$, the covariance $\Sig_j$ has a block diagonal structure
$\Sig_j = \blkdiag(\Sig_{1,j},\Sig_{2,j},\dotsc,\Sig_{\ell,j})$ (a special
case is an axis-aligned Gaussian).
The various blocks correspond to the $\ell$ different views of the data
$\x_1,\x_2,\dotsc,\x_\ell \in \R^d$ (for $d = D / \ell$), which are
conditionally independent given $h$.
The mean vector for each component $j$ is similarly partitioned into the
views as $\v\mu_j = (\v\mu_{1,j},\v\mu_{2,j},\dotsc,\v\mu_{\ell,j})$.
Note that in the case of an axis-aligned Gaussian, each covariance matrix
$\Sig_j$ is diagonal, and therefore the original coordinates $[D]$ can be
partitioned into $\ell = O(D/k)$ views (each of dimension $d = \Omega(k)$)
in any way (say, randomly) provided that Condition~\ref{cond:general}
holds.

Condition~\ref{cond:general} requires that the conditional mean matrix $M_v
= [\v\mu_{v,1} | \v\mu_{v,2} | \dotsb | \v\mu_{v,k}]$ for each view $v$
have full column rank.
This is similar to the non-degeneracy and spreading conditions used in
previous studies of multi-view clustering~\citep{CR08,CKLS09}.
In these previous works, the multi-view and non-degeneracy assumptions are
shown to reduce the minimum separation required for various efficient
algorithms to learn the model parameters.
In comparison, \algorithmB\ does not require a minimum separation condition
at all.
See Appendix~\ref{appendix:product} for details.

While \algorithmB\ recovers just the means of the mixture components (see
Appendix~\ref{appendix:subgaussian} for details concerning
Condition~\ref{cond:concentration}), we remark that a slight variation can
be used to recover the covariances as well.
Note that
\[
\E[\x_v \otimes \x_v|h]
= (M_v \e_h) \otimes (M_v \e_h) + \Sig_{v,h}
= \v\mu_{v,h} \otimes \v\mu_{v,h} + \Sig_{v,h}
\]
for all $v \in [\ell]$.
For a pair of vectors $\v\phi \in \R^d$ and $\v\psi \in \R^d$, define the
matrix $Q_{1,2,3}(\v\phi,\v\psi) \in \R^{d \times d}$ of fourth-order
moments by
$Q_{1,2,3}(\v\phi,\v\psi)
:= \E[(\x_1 \otimes \x_2) \dotp{\v\phi,\x_3} \dotp{\v\psi,\x_3}]$.
\begin{proposition} \label{proposition:gmm-cov}
Under the setting of Lemma~\ref{lemma:general-operator}, the matrix given
by
\[F_{1,2,3}(\v\phi,\v\psi)
:= (U_1^\t Q_{1,2,3}(\v\phi,\v\psi) U_2) (U_1^\t P_{1,2} U_2)^{-1}\]
satisfies
$F_{1,2,3}(\v\phi,\v\psi)
= (U_1^\t M_1) \diag(
\dotp{\v\phi,\v\mu_{3,t}}
\dotp{\v\psi,\v\mu_{3,t}}
+ \dotp{\v\phi,\Sig_{3,t}\v\psi}
: t \in [k]
)
(U_1^\t M_1)^{-1}$
and hence is diagonalizable (in fact, by the same matrices as
$B_{1,2,3}(\v\eta)$).
\end{proposition}

Finally, we note that even if Condition~\ref{cond:general} does not hold
(\emph{e.g.}, if $\v\mu_{v,j} \equiv \v{m} \in \R^d$ (say) for all $v \in
[\ell], j \in [k]$ so all of the Gaussians have the same mean), one may
still apply \algorithmB\ to the model $(h,\y_1,\y_2,\dotsc,\y_\ell)$ where
$\y_v \in \R^{d + d(d+1)/2}$ is the random vector that include both first-
and second-order terms of $\x_v$, \emph{i.e.}, $\y_v$ is the concatenation
of $x_v$ and the upper triangular part of $\x_v \otimes \x_v$.
In this case, Condition~\ref{cond:general} is replaced by a requirement
that the matrices
\[ M_v' := \left[ \begin{array}{c|c|c|c} \E[\y_v | h = 1] & \E[\y_v | h =
2] & \dotsb & \E[\y_v | h = k] \end{array} \right] \in \R^{(d + d(d+1)/2)
\times k} \]
of conditional means and covariances have full rank.
This requirement can be met even if the means $\v\mu_{v,j}$ of the mixture
components are all the same.

\subsection{Hidden Markov models}
\label{section:hmm}

A hidden Markov model is a latent variable model in which a hidden state
sequence $h_1,h_2,\dotsc,h_\ell$ forms a Markov chain $h_1 \to h_2 \to
\dotsb \to h_\ell$ over $k$ possible states $[k]$; and given the state
$h_t$ at time $t \in [k]$, the observation $\x_t$ at time $t$ (a random
vector taking values in $\R^d$) is conditionally independent of all other
observations and states.
The directed graphical model is depicted in
Figure~\ref{fig:graphical-model}(b).

The vector $\v\pi \in \Delta^{k-1}$ is the initial state distribution:
\[ \Pr[h_1 = i] = \pi_i , \quad i \in [k] . \]
For simplicity, we only consider time-homogeneous HMMs, although it is
possible to generalize to the time-varying setting.
The matrix $T \in \R^{k \times k}$ is a stochastic matrix describing the
hidden state Markov chain:
\[ \Pr[h_{t+1} = i | h_t = j] = T_{i,j} , \quad i,j \in [k], t \in [\ell-1]
.
\]
Finally, the columns of the matrix $O = [\v{o}_1 | \v{o}_2 | \dotsb |
\v{o}_k] \in \R^{d \times k}$ are the conditional means of the observation
$\x_t$ at time $t$ given the corresponding hidden state $h_t$:
\[ \E[ \x_t | h_t = i ] = O\e_i = \v{o}_i , \quad i \in [k], t \in [\ell] .
\]

Note that both discrete and continuous observations are readily handled in
this framework.
For instance, the conditional distribution of $\x_t$ given $h_t = i$ (for
$i \in [k]$) could be a high-dimensional multivariate Gaussian
with mean $\v{o}_i \in \R^d$.
Such models were not handled by previous
methods~\citep{Chang96,MR06,HKZ09}.

The restriction of the HMM to three time steps, say $t \in \{1,2,3\}$, is
an instance of the three-view mixture model.
\begin{proposition} \label{proposition:hmm-rep}
If the hidden variable $h$ (from the three-view mixture model of
Section~\ref{section:general-setting}) is identified with the second hidden
state $h_2$, then $\{\x_1, \x_2, \x_3\}$ are conditionally independent
given $h$, and the parameters of the resulting three-view mixture model on
$(h,\x_1,\x_2,\x_3)$ are
\begin{alignat*}{6}
\v{w} & := T\v\pi ,
& \qquad & M_1 := O \diag(\v\pi) T^\t \diag(T\v\pi)^{-1} ,
& \qquad & M_2 & := O ,
& \qquad & M_3 := OT .
\end{alignat*}
\end{proposition}
From Proposition~\ref{proposition:hmm-rep}, it is easy to verify that
$B_{3,1,2}(\v\eta) = (U_3^\t O T) \diag(O^\t\v\eta) (U_3^\t O T)^{-1}$.
Therefore, after recovering the observation conditional mean matrix $O$
using \algorithmB, the Markov chain transition matrix can be recovered
using the matrix of right eigenvectors $R$ of $B_{3,1,2}(\v\eta)$ and the
equation $(U_3^\t O)^{-1} R = T$ (up to scaling of the columns).

\subsubsection*{Acknowledgments}

We thank Kamalika Chaudhuri and Tong Zhang for many useful discussions,
Karl Stratos for comments on an early draft, David Sontag and an
anonymous reviewer for some pointers to related work, and Adel Javanmard
for pointing out a problem with Theorem~\ref{theorem:incoherence} in an
earlier version of the paper.

\bibliography{moments}
\bibliographystyle{plainnat}

\appendix

\section{Analysis of \algorithmA}
\label{appendix:topic}

In this appendix, we give an analysis of \algorithmA\ (but defer most
perturbation arguments to Appendix~\ref{appendix:perturb}), and also
present some illustrative empirical results on text data using a modified
implementation.

\subsection{Accuracy of moment estimates}

\begin{lemma} \label{lemma:topic-concentration}
Fix $\delta \in (0,1)$.
Let $\wh\pairs$ be the empirical average of $N$ independent copies of $\x_1
\otimes \x_2$, and let $\wh\triples$ be the empirical average of $N$
independent copies of $(\x_1 \otimes \x_2) \dotp{\v\eta,\x_3}$.
Then
\begin{enumerate}
\item $\displaystyle\Pr\Biggl[ \|\wh\pairs - \pairs\|_\F \leq \frac{1 +
\sqrt{\ln(1/\delta)}}{\sqrt{N}} \Biggr] \geq 1-\delta$, and

\item $\displaystyle\Pr\Biggl[ \forall \v\eta\in\R^d,\
\|\wh\triples(\v\eta) - \triples(\v\eta)\|_\F \leq \frac{\|\v\eta\|_2 ( 1 +
\sqrt{\ln(1/\delta)})}{\sqrt{N}} \Biggr] \geq 1-\delta$.

\end{enumerate}
\end{lemma}
\begin{proof}
The first claim follows from applying Lemma~\ref{lemma:discrete} to the
vectorizations of $\wh\pairs$ and $\pairs$ (whereupon the Frobenius norm is
the Euclidean norm of the vectorized matrices).
For the second claim, we also apply Lemma~\ref{lemma:discrete} to
$\wh\triples$ and $\triples$ in the same way to obtain, with probability at
least $1-\delta$,
\[
\sum_{i=1}^d \sum_{j=1}^d \sum_{x=1}^d (\wh\triples_{i,j,x} -
\triples_{i,j,x})^2 \leq
\frac{(1+ \sqrt{\ln(1/\delta)})^2}{N}
.
\]
Now condition on this event.
For any $\v\eta = (\eta_1,\eta_2,\dotsc,\eta_d) \in \R^d$,
\begin{align*}
\|\wh\triples(\v\eta) - \triples(\v\eta)\|_\F^2
& =
\sum_{i=1}^d \sum_{j=1}^d \Biggl|
\sum_{x=1}^d \eta_x (\wh\triples_{i,j,x} - \triples_{i,j,x})
\Biggr|^2
\\
& \leq
\sum_{i=1}^d \sum_{j=1}^d
\|\v\eta\|_2^2
\sum_{x=1}^d (\wh\triples_{i,j,x} - \triples_{i,j,x})^2
\\
& \leq \frac{\|\v\eta\|_2^2 (1 + \sqrt{\ln(1/\delta)})^2}{N}
\end{align*}
where the first inequality follows by Cauchy-Schwarz.
\end{proof}

\subsection{Proof of Theorem~\ref{theorem:topic}}
\label{appendix:topic-proof}

Let $E_1$ be the event in which
\begin{equation} \label{eq:pairs-error}
\|\wh\pairs - \pairs\|_2 \leq \frac{1 + \sqrt{\ln(1/\delta)}}{\sqrt{N}}
\end{equation}
and
\begin{equation} \label{eq:triples-error}
\|\wh\triples(\v{v}) - \triples(\v{v})\|_2 \leq \frac{\|v\|_2(1 +
\sqrt{\ln(1/\delta)})}{\sqrt{N}}
\end{equation}
for all $\v{v} \in \R^d$.
By Lemma~\ref{lemma:topic-concentration}, a union bound, and the fact that
$\|A\|_2 \leq \|A\|_\F$, we have $\Pr[E_1] \geq 1 - 2\delta$.
Now condition on $E_1$, and let $E_2$ be the event in which
\begin{equation} \label{eq:gamma}
\gamma :=
\min_{i \neq j} |\dotp{\h{U}\v\theta,M(\e_i-\e_j)}| = \min_{i \neq j}
|\dotp{\v\theta,\h{U}^\t M(\e_i-\e_j)}| > \frac{\sqrt2 \sigma_k(\h{U}^\t M)
\cdot \delta}{\sqrt{ek} {k \choose 2}} .
\end{equation}
By Lemma~\ref{lemma:eigen-gap} and the fact $\|\h{U}^\t M(\e_i-\e_j)\|_2
\geq \sqrt{2}\sigma_k(\h{U}^\t M)$, we have $\Pr[E_2|E_1] \geq 1-\delta$,
and thus $\Pr[E_1 \cap E_2] \geq (1-2\delta)(1-\delta) \geq 1-3\delta$.
So henceforth condition on this joint event $E_1 \cap E_2$.

Let $\veps_0 := \frac{\|\wh\pairs - \pairs\|_2}{\sigma_k(\pairs)}$,
$\veps_1 := \frac{\veps_0}{1-\veps_0}$,
and
$\veps_2 := \frac{\veps_0}{(1-\veps_1^2) \cdot (1- \veps_0 - \veps_1^2)}$.
The conditions on $N$ and the bound in~\eqref{eq:pairs-error} implies that
$\veps_0 < \frac1{1+\sqrt2} \leq \frac12$, so
Lemma~\ref{lemma:matrix-perturb} implies that
$\sigma_k(\h{U}^\t M) \geq \sqrt{1-\veps_1^2} \cdot \sigma_k(M)$,
$\kappa(\h{U}^\t M) \leq \frac{\|M\|_2}{\sqrt{1-\veps_1^2} \cdot
\sigma_k(M)}$, and that $\h{U}^\t \pairs \h{V}$ is invertible.
By Lemma~\ref{lemma:topic-operator},
\[ \tl{B}(\v\eta)
:= (\h{U}^\t \triples(\v\eta) \h{V}) (\h{U}^\t \pairs \h{V})^{-1}
= (\h{U}^\t M) \diag(M^\t \v\eta) (\h{U}^\t M)^{-1}. \]
Thus, Lemma~\ref{lemma:op-perturb} implies
\begin{equation} \label{eq:op-error}
\|\h{B}(\v\eta) - \tl{B}(\v\eta)\|_2
\leq \frac{\|\wh\triples(\v\eta)-\triples(\v\eta)\|_2}{(1 - \veps_0) \cdot
\sigma_k(\pairs)} + \frac{\veps_2}{\sigma_k(\pairs)}
.
\end{equation}
Let $R := \h{U}^\t M \diag(\|\h{U}^\t M\e_1\|_2, \|\h{U}^\t M\e_2\|_2,
\dotsc, \|\h{U}^\t M\e_k\|_2)^{-1}$ and $\veps_3 := \frac{\|\h{B}(\v\eta)
- \tl{B}(\v\eta)\|_2 \cdot \kappa(R)}{\gamma}$.
Note that $R$ has unit norm columns, and that $R^{-1} \tl{B}(\v\eta) R =
\diag(M^\t \v\eta)$.
By Lemma~\ref{lemma:normalize-eig} and the fact $\|M\|_2 \leq \sqrt{k}
\|M\|_1 = \sqrt{k}$,
\begin{equation} \label{eq:rinverse}
\|R^{-1}\|_2 \leq \kappa(\h{U}^\t M)
\leq \frac{\|M\|_2}{\sqrt{1-\veps_1^2} \cdot \sigma_k(M)}
\leq \frac{\sqrt{k}}{\sqrt{1-\veps_1^2} \cdot \sigma_k(M)}
\end{equation}
and
\begin{equation} \label{eq:rcond}
\kappa(R)
\leq \kappa(\h{U}^\t M)^2
\leq \frac{k}{(1-\veps_1^2) \cdot \sigma_k(M)^2}
.
\end{equation}
The conditions on $N$ and the bounds in~\eqref{eq:pairs-error},
\eqref{eq:triples-error}, \eqref{eq:gamma}, \eqref{eq:op-error},
and~\eqref{eq:rcond} imply that $\veps_3 < \frac12$.
By Lemma~\ref{lemma:eig-perturb}, there exists a permutation $\tau$ on
$[k]$ such that, for all $j \in [k]$,
\begin{equation} \label{eq:eigvector-error}
\|s_j \h\xi_j - \h{U}^\t \v\mu_{\tau(j)} / c_j' \|_2
= \|s_j \h\xi_j - R\e_{\tau(j)}\|_2 \leq 4k \cdot \|R^{-1}\|_2 \cdot \veps_3
\end{equation}
where
$s_j := \sign(\dotp{\h\xi_j, \h{U}^\t \v\mu_{\tau(j)}})$ and
$c_j' := \|\h{U}^\t \v\mu_{\tau(j)}\|_2 \leq \|\v\mu_{\tau(j)}\|_2$
(the eigenvectors $\h\xi_j$ are unique up to sign $s_j$ because each
eigenvalue has geometric multiplicity $1$).
Since $\v\mu_{\tau(j)} \in \range(U)$, Lemma~\ref{lemma:matrix-perturb} and
the bounds in~\eqref{eq:eigvector-error} and~\eqref{eq:rinverse} imply
\begin{align}
\|s_j\h{U}\h\xi_j - \v\mu_{\tau(j)} / c_j'\|_2
& \leq \sqrt{\|s_j\h\xi_j - \h{U}^\t \v\mu_{\tau(j)} / c_j'\|_2^2 +
\|\v\mu_{\tau(j)} / c_j'\|_2^2 \cdot \veps_1^2}
\nonumber
\\
& \leq \|s_j\h\xi_j - \h{U}^\t \v\mu_{\tau(j)} / c_j'\|_2 +
\|\v\mu_{\tau(j)} / c_j'\|_2 \cdot \veps_1
\nonumber
\\
& \leq 4k \cdot \|R^{-1}\|_2 \cdot \veps_3 + \veps_1
\nonumber
\\
& \leq 4k \cdot \frac{\sqrt{k}}{\sqrt{1-\veps_1^2} \cdot \sigma_k(M)}
\cdot \veps_3 + \veps_1
\nonumber
.
\end{align}
Therefore, for $c_j := s_j c_j'\dotp{\v1,\h{U} \h\xi_j}$, we have
\[
\|c_j \h\mu_j - \v\mu_{\tau(j)}\|_2
= \|c_j' s_j \h{U}\h\xi_j - \v\mu_{\tau(j)}\|_2
\leq \|\v\mu_{\tau(j)}\|_2 \cdot \biggl(
4k \cdot \frac{\sqrt{k}}{\sqrt{1-\veps_1^2} \cdot \sigma_k(M)}
\cdot \veps_3 + \veps_1
\biggr)
.
\]
Making all of the substitutions into the above bound gives
\begin{align*}
\frac{\|c_j \h\mu_j - \v\mu_{\tau(j)}\|_2}{\|\v\mu_{\tau(j)}\|_2}
& \leq \frac{4k^{1.5}}{\sqrt{1-\veps_1^2} \cdot \sigma_k(M)}
\cdot \frac{k}{(1-\veps_1^2) \cdot \sigma_k(M)^2}
\cdot \frac{\sqrt{ek} \cdot {k \choose 2}}{\sqrt{2(1-\veps_1^2)} \cdot
\sigma_k(M) \cdot \delta}
\\
& \quad{}
\cdot \biggl(
\frac{\|\wh\triples(\v\eta) - \triples(\v\eta)\|_2}{(1-\veps_0) \cdot \sigma_k(\pairs)}
+ \frac{\|\wh\pairs - \pairs\|_2}{(1-\veps_1^2) \cdot (1-\veps_0-\veps_1^2)
\cdot \sigma_k(\pairs)^2}
\biggr)
\\
& \quad{}
+ \frac{\|\wh\pairs - \pairs\|_2}{(1-\veps_0) \cdot \sigma_k(\pairs)}
\\
& \leq C \cdot \frac{k^5}{\sigma_k(M)^4 \cdot \sigma_k(\pairs)^2 \cdot
\delta} \cdot \sqrt{\frac{\ln(1/\delta)}{N}}
.
\end{align*}

\subsection{Some illustrative empirical results}

As a demonstration of feasibility, we applied a modified version of
\algorithmA\ to a subset of articles from the ``20 Newsgroups'' dataset,
specifically those in \texttt{comp.graphics}, \texttt{rec.sport.baseball},
\texttt{sci.crypt}, and \texttt{soc.religion.christian}, where
$\x_1,\x_2,\x_3$ represent three words from the beginning (first third),
middle (middle third), and end (last third) of an article.
We used $k = 25$ (although results were similar for $k \in
\{10,15,20,25,30\}$) and $d = 5441$ (after removing a standard set of $524$
stop-words and applying Porter stemming).
Instead of using a single $\v\eta$ and extracting all eigenvectors of
$\h{B}(\v\eta)$, we extracted a single eigenvector $\v\xi_x$ from
$\h{B}(\e_x)$ for several words $x \in [d]$ (these $x$'s were chosen using
an automatic heuristic based on their statistical leverage scores in
$\wh\pairs$).
Below, for each such $(\h{B}(\e_x),\v\xi_x)$, we report the top $15$ words
$y$ ordered by $\e_y^\t \h{U}\v\xi_x$ value.

\begin{center}
{
\begin{tabular}{|c|c|c|c|c|c|}
\hline
$\h{B}(\e_{\text{format}})$
& $\h{B}(\e_{\text{god}})$
& $\h{B}(\e_{\text{key}})$
& $\h{B}(\e_{\text{polygon}})$
& $\h{B}(\e_{\text{team}})$
& $\h{B}(\e_{\text{today}})$
\\
\hline
source    & god       & key        & polygon   & win     & game     \\
find      & write     & bit        & time      & game    & tiger    \\
post      & jesus     & chip       & save      & run     & bit      \\
image     & christian & system     & refer     & team    & run      \\
feal      & christ    & encrypt    & book      & year    & pitch    \\
intersect & people    & car        & source    & don     & day      \\
email     & time      & repository & man       & watch   & team     \\
rpi       & apr       & ve         & routine   & good    & true     \\
time      & sin       & public     & netcom    & score   & lot      \\
problem   & bible     & escrow     & gif       & yankees & book     \\
file      & day       & secure     & record    & pitch   & lost     \\
program   & church    & make       & subscribe & start   & colorado \\
gif       & person    & clipper    & change    & bit     & fan      \\
bit       & book      & write      & algorithm & time    & apr      \\
jpeg      & life      & nsa        & scott     & wonder  & watch    \\
\hline
\end{tabular}
}
\end{center}
The first and fourth topics appear to be about computer graphics
(\texttt{comp.graphics}), the fifth and sixth about baseball
(\texttt{rec.sports.baseball}), the third about encryption
(\texttt{sci.crypt}), and the second about Christianity
(\texttt{soc.religion.christian}).

We also remark that \algorithmA\ can be implemented so that it makes just
two passes over the training data, and that simple hashing or random
projection tricks can reduce the memory requirement to $O(k^2 + kd)$
(\emph{i.e.}, $\wh\pairs$ and $\wh\triples$ never need to be explicitly
formed).

\section{Proofs and details from Section~\ref{section:general}}
\label{appendix:general}

In this section, we provide ommitted proofs and discussion from
Section~\ref{section:general}.

\subsection{Proof of Lemma~\ref{lemma:general-moments}}
By conditional independence,
\begin{align*}
P_{1,2}
= \E[\E[\x_1 \otimes \x_2|h]]
& = \E[\E[\x_1|h] \otimes \E[\x_2|h]] \\
& = \E[(M_1 \e_h) \otimes (M_2 \e_h)]
= M_1 \biggl( \sum_{t=1}^k w_t \e_t \otimes \e_t \biggr) M_2^\t
= M_1 \diag(\v{w}) M_2^\t
.
\end{align*}
Similarly,
\begin{align*}
P_{1,2,3}(\v\eta)
& = \E[\E[(\x_1 \otimes \x_2) \dotp{\v\eta,\x_3}|h]]
= \E[\E[\x_1|h] \otimes \E[\x_2|h] \dotp{\v\eta,\E[\x_3|h]}] \\
& = \E[(M_1 \e_h) \otimes (M_2 \e_h) \dotp{\v\eta,M_3\e_h}]
= M_1 \biggl( \sum_{t=1}^k w_t \e_h \otimes \e_h \dotp{\v\eta,M_3\e_h}
\biggr) M_2^\t \\
& = M_1 \diag(M_3^\t\v\eta) \diag(\v{w}) M_2^\t
.
\end{align*}

\subsection{Proof of Lemma~\ref{lemma:general-operator}}

We have $U_1^\t P_{1,2} U_2 = (U_1^\t M_1) \diag(\v{w}) (M_2^\t U_2)$ by
Lemma~\ref{lemma:general-moments}, which is invertible by the assumptions
on $U_v$ and Condition~\ref{cond:general}.
Moreover, also by Lemma~\ref{lemma:general-moments},
\begin{align*}
B_{1,2,3}(\v\eta)
& = (U_1^\t P_{1,2,3}(\v\eta) U_2) \ (U_1^\t P_{1,2} U_2)^{-1} \\
& = (U_1^\t M_1 \diag(M_3^\t\v\eta) \diag(\v{w}) M_2^\t U_2) \ (U_1^\t P_{1,2} U_2)^{-1} \\
& = (U_1^\t M_1) \diag(M_3^\t\v\eta) (U_1^\t M_1)^{-1}
\ (U_1^\t M_1 \diag(\v{w}) M_2^\t U_2) \ (U_1^\t P_{1,2} U_2)^{-1} \\
& = (U_1^\t M_1) \diag(M_3^\t\v\eta) (U_1^\t M_1)^{-1}
\ (U_1^\t P_{1,2} U_2)
\ (U_1^\t P_{1,2} U_2)^{-1} \\
& = (U_1^\t M_1) \diag(M_3^\t\v\eta) (U_1^\t M_1)^{-1}
.
\end{align*}

\subsection{Proof of Lemma~\ref{lemma:general-system}}

By Lemma~\ref{lemma:general-operator},
\begin{align*}
(U_1^\t M_1)^{-1} B_{1,2,3}(U_3\v\theta_i) (U_1^\t M_1)
& = \diag(M_3^\t U_3 \v\theta_i) \\
& = \diag(\dotp{\v\theta_i, U_3^\t M_3\e_1},
\dotp{\v\theta_i, U_3^\t M_3\e_2},\dotsc
\dotp{\v\theta_i, U_3^\t M_3\e_k}) \\
& = \diag(\lambda_{i,1},\lambda_{i,2},\dotsc,\lambda_{i,k})
\end{align*}
for all $i \in [k]$, and therefore
\[
L
= \begin{bmatrix}
\dotp{\v\theta_1,U_3^\t M_3\e_1} & \dotp{\v\theta_1,U_3^\t M_3\e_2} & \dotsb &
\dotp{\v\theta_1,U_3^\t M_3\e_k} \\
\dotp{\v\theta_2,U_3^\t M_3\e_1} & \dotp{\v\theta_2,U_3^\t M_3\e_2} & \dotsb &
\dotp{\v\theta_2,U_3^\t M_3\e_3} \\
\vdots          & \vdots          & \ddots & \vdots \\
\dotp{\v\theta_k,U_3^\t M_3\e_1} & \dotp{\v\theta_k,U_3^\t M_3\e_2} & \dotsb &
\dotp{\v\theta_k,U_3^\t M_3\e_k}
\end{bmatrix}
= \Theta U_3^\t M_3
.
\]

\subsection{Ordering issues}
\label{appendix:ordering}

Although \algorithmB\ only explicitly yields estimates for $M_3$, it can
easily be applied to estimate $M_v$ for all other views $v$.
The main caveat is that the estimators may not yield the same ordering of
the columns, due to the unspecified order of the eigenvectors obtained in
the third step of the method, and therefore some care is needed to obtain a
consistent ordering.
However, this ordering issue can be handled by exploiting consistency
across the multiple views.

The first step is to perform the estimation of $M_3$ using \algorithmB\ as
is.
Then, to estimate $M_2$, one may re-use the eigenvectors in $\h{R}_1$ to
diagonalize $\h{B}_{1,3,2}(\v\eta)$, as $B_{1,2,3}(\v\eta)$ and
$B_{1,3,2}(\v\eta)$ share the same eigenvectors.
The same goes for estimating $M_v$ for other all other views $v$ except $v
= 1$.

It remains to provide a way to estimate $M_1$.
Observe that $M_2$ can be estimated in at least two ways: via the operators
$\h{B}_{1,3,2}(\v\eta)$, or via the operators
$\h{B}_{3,1,2}(\v\eta)$.
This is because the eigenvalues of $B_{3,1,2}(\v\eta)$ and
$B_{1,3,2}(\v\eta)$ are the identical.
Because the eigenvalues are also sufficiently separated from each other,
the eigenvectors $\h{R}_3$ of $\h{B}_{3,1,2}(\v\eta)$ can be put in the
same order as the eigenvectors $\h{R}_1$ of $\h{B}_{1,3,2}$ by
(approximately) matching up their respective corresponding eigenvalues.
Finally, the appropriately re-ordered eigenvectors $\h{R}_3$ can then be
used to diagonalize $\h{B}_{3,2,1}(\v\eta)$ to estimate $M_1$.

\subsection{Estimating the mixing weights}
\label{appendix:mixing-weights}

Given the estimate of $\h{M}_3$, one can obtain an estimate of $\v{w}$
using
\[ \h{w} := \h{M}_3^\dag \h\E[\x_3] \]
where $A^\dag$ denotes the Moore-Penrose pseudoinverse of $A$ (though other
generalized inverses may work as well), and $\h\E[\x_3]$ is the empirical
average of $\x_3$.
This estimator is based on the following observation:
\[ \E[\x_3] = \E[\E[\x_3|h]] = M_3\E[\e_h] = M_3\v{w} \]
and therefore
\[ M_3^\dag \E[\x_3] = M_3^\dag M_3 \v{w} = \v{w} \]
since $M_3$ has full column rank.

\subsection{Proof of Theorem~\ref{theorem:general}}
\label{appendix:general-proof}

The proof is similar to that of Theorem~\ref{theorem:topic}, so we just
describe the essential differences.
As before, most perturbation arguments are deferred to
Appendix~\ref{appendix:perturb}.

First, let $E_1$ be the event in which
\begin{align*}
\|\h{P}_{1,2} - P_{1,2}\|_2
& \leq C_{1,2} \cdot f(N,\delta) 
,
\\
\|\h{P}_{1,3} - P_{1,3}\|_2
& \leq C_{1,3} \cdot f(N,\delta) 
\end{align*}
and
\[
\|\h{P}_{1,2,3}(\h{U}_3\v\theta_i) - P_{1,2,3}(\h{U}_3\v\theta_i)\|_2 \leq
C_{1,2,3} \cdot f(N,\delta/k)
\]
for all $i \in [k]$.
Therefore by Condition~\ref{cond:concentration} and a union bound, we have
$\Pr[E_1] \geq 1-3\delta$.
Second, let $E_2$ be the event in which
\[
\gamma :=
\min_{i \in [k]}
\min_{j \neq j'}
|\dotp{\v\theta_i,\h{U}_3^\t M_3(\e_j-\e_{j'})}|
> \frac{\min_{j \neq j'} \|\h{U}_3^\t M_3(\e_j - \e_{j'})\|_2 \cdot
\delta}{\sqrt{ek} {k \choose 2} k}
\]
and
\[
\lambda_{\max} :=
\max_{i,j \in [k]}
|\dotp{\v\theta_i,\h{U}_3^\t M_3\e_j}|
\leq \frac{\max_{j \in [k]} \|M_3\e_j\|_2}{\sqrt{k}}
\Bigl(1 + \sqrt{2\ln(k^2/\delta)} \Bigr)
.
\]
Since each $\v\theta_i$ is distributed uniformly over $\sphere^{k-1}$, it
follows from Lemma~\ref{lemma:eigen-gap} and a union bound that
$\Pr[E_2|E_1] \geq 1-2\delta$.
Therefore $\Pr[E_1 \cap E_2] \geq (1-3\delta)(1-2\delta) \geq 1-5\delta$.

Let $U_3 \in \R^{d \times k}$ be the matrix of top $k$ orthonormal left
singular vectors of $M_3$.
By Lemma~\ref{lemma:matrix-perturb} and the conditions on $N$, we have
$\sigma_k(\h{U}_3^\t U_3) \geq 1/2$, and therefore
\[ \gamma \geq \frac{\min_{i \neq i'} \|M_3(\e_i - \e_{i'})\|_2 \cdot
\delta}{2\sqrt{ek}{k \choose 2} k}
\quad\text{and}\quad
\frac{\lambda_{\max}}{\gamma}
\leq \frac{\sqrt{e} k^3 (1 + \sqrt{2\ln(k^2/\delta)})}{\delta}
\cdot \kappa'(M_3)
\]
where
\[
\kappa'(M_3)
:=
\frac{\max_{i \in [m]} \|M_3\e_i\|_2}
{\min_{i \neq i'} \|M_3(\e_i - \e_{i'})\|_2}
.
\]

Let $\v\eta_i := \h{U}_3\v\theta_i$ for $i \in [k]$.
By Lemma~\ref{lemma:matrix-perturb}, $\h{U}_1^\t P_{1,2} \h{U}_2$ is
invertible, so we may define $\tl{B}_{1,2,3}(\v\eta_i) := (\h{U}_1^\t
P_{1,2,3}(\v\eta_i) \h{U}_2) (\h{U}_1^\t P_{1,2} \h{U}_2)^{-1}$.
By Lemma~\ref{lemma:general-operator},
\[
\tl{B}_{1,2,3}(\v\eta_i)
= (\h{U}_1^\t M_1) \diag(M_3^\t\v\eta_i) (\h{U}_1^\t M_1)^{-1}
.
\]
Also define $R := \h{U}_1^\t M_1 \diag(\|\h{U}_1^\t M_1\e_1\|_2,
\|\h{U}_1^\t M_1\e_2\|_2, \dotsc, \|\h{U}_1^\t M_1\e_k\|_2)^{-1}$.
Using most of the same arguments in the proof of
Theorem~\ref{theorem:topic}, we have
\begin{align}
\|R^{-1}\|_2 & \leq 2\kappa(M_1)
,
\label{eq:rinverse-general}
\\
\kappa(R) & \leq 4\kappa(M_1)^2
,
\label{eq:rcond-general}
\\
\|\h{B}_{1,2,3}(\v\eta_i) - \tl{B}_{1,2,3}(\v\eta_i)\|_2
& \leq
\frac{2\|\h{P}_{1,2,3}(\v\eta_i) - P_{1,2,3}(\v\eta_i)\|_2}{\sigma_k(P_{1,2})}
+ \frac{2\|P_{1,2,3}\|_2 \cdot \|\h{P}_{1,2} - P_{1,2}\|_2}{\sigma_k(P_{1,2})^2}
.
\nonumber
\end{align}
By Lemma~\ref{lemma:eig-perturb}, the operator $\h{B}_{1,2,3}(\v\eta_1)$
has $k$ distinct eigenvalues, and hence its matrix of right eigenvectors
$\h{R}_1$ is unique up to column scaling and ordering.
This in turn implies that $\h{R}_1^{-1}$ is unique up to row scaling and
ordering.
Therefore, for each $i \in [k]$, the $\h\lambda_{i,j} = \e_j^\t
\h{R}_1^{-1} \h{B}_{1,2,3}(\v\eta_i) \h{R}_1 \e_j$ for $j \in [k]$ are
uniquely defined up to ordering.
Moreover, by Lemma~\ref{lemma:eig-perturb-all} and the above bounds on
$\|\h{B}_{1,2,3}(\v\eta_i)-\tl{B}_{1,2,3}(\v\eta_i)\|_2$ and $\gamma$,
there exists a permutation $\tau$ on $[k]$ such that, for all $i,j \in
[k]$,
\begin{align}
|\h\lambda_{i,j} - \lambda_{i,\tau(j)}|
& \leq \Bigl( 3\kappa(R) + 16k^{1.5} \cdot \kappa(R) \cdot \|R^{-1}\|_2^2
\cdot \lambda_{\max}/\gamma \Bigr)
\cdot \|\h{B}_{1,2,3}(\v\eta_i) -\tl{B}_{1,2,3}(\v\eta_i)\|_2
\nonumber
\\
& \leq \Bigl( 12\kappa(M_1)^2 +
256k^{1.5} \cdot \kappa(M_1)^4 \cdot \lambda_{\max}/\gamma \Bigr)
\cdot \|\h{B}_{1,2,3}(\v\eta_i) -\tl{B}_{1,2,3}(\v\eta_i)\|_2
\label{eq:lambda-error}
\end{align}
where the second inequality uses~\eqref{eq:rinverse-general}
and~\eqref{eq:rcond-general}.
Let $\h\nu_j := (\h\lambda_{1,j}, \h\lambda_{2,j}, \dotsc, \h\lambda_{k,j})
\in \R^k$ and $\v\nu_j := (\lambda_{1,j}, \lambda_{2,j}, \dotsc,
\lambda_{k,j}) \in \R^k$.
Observe that $\v\nu_j = \Theta \h{U}_3^\t M_3\e_j = \Theta\h{U}_3^\t
\v\mu_{3,j}$ by Lemma~\ref{lemma:general-system}.
By the orthogonality of $\Theta$, the fact $\|\v{v}\|_2 \leq \sqrt{k}
\|\v{v}\|_\infty$ for $\v{v} \in \R^k$, and~\eqref{eq:lambda-error}
\begin{align}
\|\Theta^{-1}\h\nu_j - \h{U}_3^\t \v\mu_{3,\tau(j)}\|_2
& = \|\Theta^{-1}(\h\nu_j - \v\nu_{\tau(j)})\|_2
\nonumber \\
& = \|\h\nu_j - \v\nu_{\tau(j)}\|_2
\nonumber \\
& \leq \sqrt{k} \cdot \|\h\nu_j - \v\nu_{\tau(j)}\|_\infty
\nonumber \\
& = \sqrt{k} \cdot \max_i |\h\lambda_{i,j} - \lambda_{i,\tau(j)}|
\nonumber \\
& \leq
\Bigl( 12\sqrt{k} \cdot \kappa(M_1)^2 +
256k^2 \cdot \kappa(M_1)^4 \cdot \lambda_{\max}/\gamma \Bigr)
\cdot \|\h{B}_{1,2,3}(\v\eta_i) -\tl{B}_{1,2,3}(\v\eta_i)\|_2
\nonumber
.
\end{align}
Finally, by Lemma~\ref{lemma:matrix-perturb} (as applied to $\h{P}_{1,3}$
and $P_{1,3}$),
\[
\|\h\mu_{3,j} - \v\mu_{3,\tau(j)}\|_2
\leq \|\Theta^{-1}\h\nu_j - \h{U}_3^\t \v\mu_{3,\tau(j)}\|_2
+ 2\|\v\mu_{3,\tau(j)}\|_2 \cdot \frac{\|\h{P}_{1,3} -
P_{1,3}\|_2}{\sigma_k(P_{1,3})}
.
\]
Making all of the substitutions into the above bound gives
\begin{align*}
\|\h\mu_{3,j} - \v\mu_{3,\tau(j)}\|_2
& \leq
\frac{C}{6} \cdot
k^5 \cdot \kappa(M_1)^4 \cdot \kappa'(M_3) \cdot \frac{\ln(k/\delta)}{\delta}
\cdot \biggl(
\frac{C_{1,2,3} \cdot f(N,\delta/k)}{\sigma_k(P_{1,2})}
+ \frac{\|P_{1,2,3}\|_2 \cdot C_{1,2} \cdot f(N/\delta)}{\sigma_k(P_{1,2})^2}
\biggr)
\\
& \qquad\qquad{}
+
\frac{C}{6} \cdot
\|\v\mu_{3,\tau(j)}\|_2 \cdot \frac{C_{1,3} \cdot
f(N,\delta)}{\sigma_k(P_{1,3})}
\\
& \leq \frac12
\Bigl(\max_{j' \in [k]} \|\v\mu_{3,j'}\|_2 + \|\v\mu_{3,\tau(j)}\|_2
\Bigr) \cdot \epsilon
\\
& \leq \max_{j' \in [k]} \|\v\mu_{3,j'}\|_2 \cdot \epsilon
.
\end{align*}

\section{Perturbation analysis for observable operators}
\label{appendix:perturb}

The following lemma establishes the accuracy of approximating the
fundamental subspaces (\emph{i.e.}, the row and column spaces) of a matrix
$X$ by computing the singular value decomposition of a perturbation $\h{X}$
of $X$.

\begin{lemma} \label{lemma:matrix-perturb}
Let $X \in \R^{m \times n}$ be a matrix of rank $k$.
Let $U \in \R^{m \times k}$ and $V \in \R^{n \times k}$ be matrices with
orthonormal columns such that $\range(U)$ and $\range(V)$ are spanned by,
respectively, the left and right singular vectors of $X$ corresponding to
its $k$ largest singular values.
Similarly define $\h{U} \in \R^{m \times k}$ and $\h{V} \in \R^{n \times
k}$ relative to a matrix $\h{X} \in \R^{m \times n}$.
Define $\eps_X := \|\h{X} - X\|_2$, $\veps_0 :=
\frac{\eps_X}{\sigma_k(X)}$, and $\veps_1 := \frac{\veps_0}{1-\veps_0}$.
Assume $\veps_0 < \frac{1}{2}$.
Then
\begin{enumerate}
\item $\veps_1 < 1$;

\item $\sigma_k(\h{X}) = \sigma_k(\h{U}^\t \h{X} \h{V}) \geq (1-\veps_0)
\cdot \sigma_k(X) > 0$;

\item $\sigma_k(\h{U}^\t U) \geq \sqrt{1-\veps_1^2}$;

\item $\sigma_k(\h{V}^\t V) \geq \sqrt{1-\veps_1^2}$;

\item $\sigma_k(\h{U}^\t X \h{V}) \geq (1-\veps_1^2) \cdot \sigma_k(X)$;

\item for any $\h\alpha \in \R^k$ and
$\v{v} \in \range(U)$,
$\|\h{U}\h\alpha - \v{v}\|_2^2 \leq \|\h\alpha - \h{U}^\t
\v{v}\|_2^2 + \|\v{v}\|_2^2 \cdot \veps_1^2$.

\end{enumerate}
\end{lemma}
\begin{proof}
The first claim follows from the assumption on $\veps_0$.
The second claim follows from the assumptions and Weyl's theorem
(Lemma~\ref{lemma:weyl}).
Let the columns of $\h{U}_\perp \in \R^{m \times (m-k)}$ be an orthonormal
basis for the orthogonal complement of $\range(\h{U})$, so that
$\|\h{U}_\perp^\t U\|_2 \leq \eps_X / \sigma_k(\h{X}) \leq \veps_1$ by
Wedin's theorem (Lemma~\ref{lemma:wedin}).
The third claim then follows because $\|\h{U}^\t U\|_2^2 = 1 -
\|\h{U}_\perp^\t U\|_2^2 \geq 1 - \veps_1^2$.
The fourth claim is analogous to the third claim, and the fifth claim
follows from the third and fourth.
The sixth claim follows writing $\v{v} = U\v\alpha$ for some $\v\alpha \in
\R^k$, and using the decomposition
$\|\h{U}\h\alpha - \v{v}\|_2^2
= \|\h{U}\h\alpha - \h{U}\h{U}^\t \v{v}\|_2^2
+ \|\h{U}_\perp\h{U}_\perp^\t \v{v}\|_2^2
= \|\h\alpha - \h{U}^\t \v{v}\|_2^2
+ \|\h{U}_\perp^\t (U\v\alpha)\|_2^2
\leq \|\h\alpha - \h{U}^\t \v{v}\|_2^2
+ \|\h{U}_\perp^\t U\|_2^2 \|\v\alpha\|_2^2
\leq \|\h\alpha - \h{U}^\t \v{v}\|_2^2
+ \|\v\alpha\|_2^2 \cdot \veps_1^2
= \|\h\alpha - \h{U}^\t U\v\alpha\|_2^2
+ \|\v{v}\|_2^2 \cdot \veps_1^2$
where the last inequality follows from the argument for the third claim,
and the last equality uses the orthonormality of the columns of $U$.
\end{proof}

The next lemma bounds the error of the observation operator in terms of the
errors in estimating the second-order and third-order moments.

\begin{lemma} \label{lemma:op-perturb}
Consider the setting and definitions from Lemma~\ref{lemma:matrix-perturb},
and let $Y \in \R^{m \times n}$ and $\h{Y} \in \R^{m \times n}$ be given.
Define $\veps_2 := \frac{\veps_0}{(1-\veps_1^2) \cdot (1- \veps_0 -
\veps_1^2)}$ and $\eps_Y := \|\h{Y} - Y\|_2$.
Assume $\veps_0 < \frac1{1+\sqrt2}$.
Then
\begin{enumerate}
\item $\h{U}^\t \h{X} \h{V}$ and $\h{U}^\t X \h{V}$ are both invertible,
and $\|(\h{U}^\t \h{X} \h{V})^{-1} - (\h{U}^\t X \h{V})^{-1}\|_2 \leq
\frac{\veps_2}{\sigma_k(X)}$;

\item $\|(\h{U}^\t \h{Y} \h{V}) (\h{U}^\t \h{X} \h{V})^{-1} - (\h{U}^\t Y
\h{V}) (\h{U}^\t X \h{V})^{-1}\|_2 \leq \frac{\eps_Y}{(1 - \veps_0) \cdot
\sigma_k(X)} + \frac{\|Y\|_2 \cdot \veps_2}{\sigma_k(X)}$.

\end{enumerate}

\end{lemma}
\begin{proof}
Let $\h{S} := \h{U}^\t \h{X} \h{V}$ and $\tl{S} := \h{U}^\t X \h{V}$.
By Lemma~\ref{lemma:matrix-perturb}, $\h{U}^\t \h{X} \h{V}$ is invertible,
$\sigma_k(\tl{S}) \geq \sigma_k(\h{U}^\t U) \cdot \sigma_k(X) \cdot
\sigma_k(\h{V}^\t V) \geq (1-\veps_1^2) \cdot \sigma_k(X)$ (so $\tl{S}$ is
also invertible), and $\|\h{S} - \tl{S}\|_2 \leq \veps_0 \cdot \sigma_k(X)
\leq \frac{\veps_0}{1-\veps_1^2} \cdot \sigma_k(\tl{S})$.
The assumption on $\veps_0$ implies $\frac{\veps_0}{1-\veps_1^2} < 1$;
therefore the Lemma~\ref{lemma:inverse-perturb} implies $\|\h{S}^{-1} -
\tl{S}^{-1}\|_2 \leq \frac{\|\h{S} - \tl{S}\|_2 / \sigma_k(\tl{S})}{1 -
\|\h{S} - \tl{S}\|_2 / \sigma_k(\tl{S})} \cdot \frac1{\sigma_k(\tl{S})}
\leq \frac{\veps_2}{\sigma_k(X)}$, which proves the first claim.
For the second claim, observe that
\begin{align*}
\lefteqn{
\|(\h{U}^\t \h{Y} \h{V}) (\h{U}^\t \h{X} \h{V})^{-1} - (\h{U}^\t Y \h{V})
(\h{U}^\t X \h{V})^{-1}\|_2
} \\
& \leq
\|(\h{U}^\t \h{Y} \h{V}) (\h{U}^\t \h{X} \h{V})^{-1} - (\h{U}^\t Y \h{V})
(\h{U}^\t \h{X} \h{V})^{-1}\|_2
+ \|(\h{U}^\t Y \h{V}) (\h{U}^\t \h{X} \h{V})^{-1} - (\h{U}^\t Y \h{V}) (\h{U}^\t X \h{V})^{-1}\|_2
\\
& \leq
\|\h{U}^\t \h{Y} \h{V} - \h{U}^\t Y \h{V}\|_2
\cdot \|(\h{U}^\t \h{X} \h{V})^{-1}\|_2
+\|\h{U}^\t Y \h{V}\|_2 \cdot \|(\h{U}^\t \h{X} \h{V})^{-1} - (\h{U}^\t X \h{V})^{-1}\|_2
\\
& \leq \frac{\eps_Y}{(1-\veps_0) \cdot \sigma_k(X)}
+ \frac{\|Y\|_2 \cdot \veps_2}{\sigma_k(X)}
\end{align*}
where the first inequality follows from the triangle inequality, the second
follows from the sub-multiplicative property of the spectral norm, and the
last follows from Lemma~\ref{lemma:matrix-perturb} and the first claim.
\end{proof}

The following lemma establishes standard eigenvalue and eigenvector
perturbation bounds.

\begin{lemma} \label{lemma:eig-perturb}
Let $A \in \R^{k \times k}$ be a diagonalizable matrix with $k$ distinct
real eigenvalues $\lambda_1,\lambda_2,\dotsc,\lambda_k \in \R$
corresponding to the (right) eigenvectors $\v\xi_1, \v\xi_2, \dotsc,
\v\xi_k \in \R^k$ all normalized to have $\|\v\xi_i\|_2 = 1$.
Let $R \in \R^{k \times k}$ be the matrix whose $i$-th column is $\v\xi_i$.
Let $\h{A} \in \R^{k \times k}$ be a matrix.
Define $\eps_A := \|\h{A} - A\|_2$,
$\gamma_A := \min_{i \neq j} |\lambda_i - \lambda_j|$,
and
$\veps_3 := \frac{\kappa(R) \cdot \eps_A}{\gamma_A}$.
Assume $\veps_3 < \frac12$.
Then there exists a permutation $\tau$ on $[k]$ such that the
following holds:
\begin{enumerate}
\item $\h{A}$ has $k$ distinct real eigenvalues $\h\lambda_1, \h\lambda_2,
\dotsc, \h\lambda_k \in \R$, and $|\h\lambda_{\tau(i)} - \lambda_i| \leq
\veps_3 \cdot \gamma_A$ for all $i \in [k]$;

\item $\h{A}$ has corresponding (right) eigenvectors $\h\xi_1, \h\xi_2,
\dotsc, \h\xi_k \in \R^k$, normalized to have $\|\h\xi_i\|_2 = 1$, which
satisfy $\|\h\xi_{\tau(i)} - \v\xi_i\|_2 \leq 4(k-1) \cdot \|R^{-1}\|_2
\cdot \veps_3$ for all $i \in [k]$;

\item the matrix $\h{R} \in \R^{k \times k}$ whose $i$-th column is
$\h\xi_{\tau(i)}$ satisfies $\|\h{R}-R\|_2 \leq \|\h{R}-R\|_\F \leq
4k^{1/2}(k-1) \cdot \|R^{-1}\|_2 \cdot \veps_3$.

\end{enumerate}

\end{lemma}
\begin{proof}
The Bauer-Fike theorem (Lemma~\ref{lemma:bauer-fike}) implies that for
every eigenvalue $\h\lambda_i$ of $\h{A}$, there exists an eigenvalue
$\lambda_j$ of $A$ such that $|\h\lambda_i - \lambda_j| \leq
\|R^{-1}(\h{A}-A)R\|_2 \leq \veps_3 \cdot \gamma_A$.
Therefore, the assumption on $\veps_3$ implies that there exists a
permutation $\tau$ such that $|\h\lambda_{\tau(i)} - \lambda_i| \leq
\veps_3 \cdot \gamma_A < \frac{\gamma_A}{2}$.
In particular,
\begin{equation} \label{eq:gerschgorin}
\Bigl| \Bigl[\lambda_i - \frac{\gamma_A}{2}, \lambda_i + \frac{\gamma_A}{2}
\Bigr] \cap \{\h\lambda_1,\h\lambda_2,\dotsc,\h\lambda_k\} \Bigr| = 1 ,
\quad \forall i \in [k]
.
\end{equation}
Since $\h{A}$ is real, all non-real eigenvalues of $\h{A}$ must come in
conjugate pairs; so the existence of a non-real eigenvalue of $\h{A}$ would
contradict~\eqref{eq:gerschgorin}.
This proves the first claim.

For the second claim, assume for notational simplicity that the permutation
$\tau$ is the identity permutation.
Let $\h{R} \in \R^{k \times k}$ be the matrix whose $i$-th column is
$\h\xi_i$.
Define $\v\zeta_i^\t \in \R^k$ to be the $i$-th row of $R^{-1}$
(\emph{i.e.}, the $i$-th left eigenvector of $A$), and similarly define
$\h\zeta_i^\t \in \R^k$ to be the $i$-th row of $\h{R}^{-1}$.
Fix a particular $i \in [k]$.
Since $\{\v\xi_1, \v\xi_2, \dotsc, \v\xi_k\}$ forms a basis for $\R^k$, we
can write $\h\xi_i = \sum_{j=1}^k c_{i,j} \v\xi_j$ for some coefficients
$c_{i,1}, c_{i,2}, \dotsc, c_{i,k} \in \R$.
We may assume $c_{i,i} \geq 0$ (or else we replace $\h\xi_i$ with
$-\h\xi_i$).
The fact that $\|\h\xi_i\|_2 = \|\v\xi_j\|_2 = 1$ for all $j \in
[k]$ and the triangle inequality imply $1 = \|\h\xi_i\|_2 \leq
c_{i,i} \|\v\xi_i\|_2 + \sum_{j \neq i} |c_{i,j}| \|\v\xi_j\|_2 = c_{i,i} +
\sum_{j \neq i} |c_{i,j}|$, and therefore
\[ \|\h\xi_i - \v\xi_i\|_2 \leq |1-c_{i,i}| \|\v\xi_i\|_2
+ \sum_{j \neq i} |c_{i,j}
\|\v\xi_j\|_2 \leq 2 \sum_{j\neq i} |c_{i,j}|
\]
again by the triangle inequality.
Therefore, it suffices to show $|c_{i,j}| \leq 2\|R^{-1}\|_2 \cdot \veps_3$
for $j \neq i$ to prove the second claim.

Observe that $A\h\xi_i = A(\sum_{i'=1}^k c_{i,i'} \v\xi_{i'}) =
\sum_{i'=1}^k c_{i,i'} \lambda_{i'} \v\xi_{i'}$, and therefore
\[ \sum_{i'=1}^k c_{i,i'} \lambda_{i'} \v\xi_{i'}
+ (\h{A} - A) \h\xi_i
= \h{A} \h\xi_i
= \h\lambda_i \h\xi_i
= \lambda_i \sum_{i'=1}^k c_{i,i'} \v\xi_{i'} 
+ (\h\lambda_i - \lambda_i) \h\xi_i
.
\]
Multiplying through the above equation by $\v\zeta_j^\t$, and using the
fact that $\v\zeta_j^\t \v\xi_{i'} = \I\{j = i'\}$ gives
\[ c_{i,j} \lambda_j + \v\zeta_i^\t (\h{A} - A) \h\xi_i
= \lambda_i c_{i,j} + (\h\lambda_i - \lambda_i) \v\zeta_j^\t \h\xi_i
.
\]
The above equation rearranges to $(\lambda_j - \lambda_i) c_{i,j} =
(\h\lambda_i - \lambda_i) \v\zeta_j^\t \h\xi_i + \v\zeta_j^\t (A - \h{A})
\h\xi_i$ and therefore
\[ |c_{i,j}|
\leq \frac{\|\v\zeta_j\|_2 \cdot (|\h\lambda_i - \lambda_i| + \|(\h{A} -
A)\h\xi_i\|_2)}{|\lambda_j - \lambda_i|}
\leq \frac{\|R^{-1}\|_2 \cdot (|\h\lambda_i - \lambda_i| + \|\h{A} -
A\|_2)}{|\lambda_j - \lambda_i|}
\]
by the Cauchy-Schwarz and triangle inequalities and the sub-multiplicative
property of the spectral norm.
The bound $|c_{i,j}| \leq 2\|R^{-1}\|_2 \cdot \veps_3$ then follows from
the first claim.

The third claim follows from standard comparisons of matrix norms.
\end{proof}

The next lemma gives perturbation bounds for estimating the eigenvalues of
simultaneously diagonalizable matrices $A_1, A_2, \dotsc, A_k$.
The eigenvectors $\h{R}$ are taken from a perturbation of the first matrix
$A_1$, and are then subsequently used to approximately diagonalize the
perturbations of the remaining matrices $A_2, \dotsc, A_k$.
In practice, one may use Jacobi-like procedures to approximately solve the
joint eigenvalue problem.

\begin{lemma} \label{lemma:eig-perturb-all}
Let $A_1, A_2, \dotsc, A_k \in \R^{k \times k}$ be diagonalizable matrices
that are diagonalized by the same matrix invertible $R \in \R^{k \times k}$
with unit length columns $\|R\e_j\|_2 = 1$, such that each $A_i$ has $k$
distinct real eigenvalues:
\[ R^{-1} A_i R = \diag(\lambda_{i,1}, \lambda_{i,2}, \dotsc,
\lambda_{i,k})
.
\]
Let $\h{A}_1, \h{A}_2, \dotsc, \h{A}_k \in \R^{k \times k}$ be given.
Define $\eps_A := \max_i \|\h{A}_i - A_i\|_2$, $\gamma_A := \min_i \min_{j
\neq j'} |\lambda_{i,j} - \lambda_{i,j'}|$,
$\lambda_{\max} := \max_{i,j} |\lambda_{i,j}|$,
$\veps_3 := \frac{\kappa(R) \cdot \eps_A}{\gamma_A}$,
and $\veps_4 := 4k^{1.5} \cdot \|R^{-1}\|_2^2 \cdot \veps_3$.
Assume $\veps_3 < \frac12$ and $\veps_4 < 1$.
Then there exists a permutation $\tau$ on $[k]$ such that the
following holds.
\begin{enumerate}
\item The matrix $\h{A}_1$ has $k$ distinct real eigenvalues
$\h\lambda_{1,1}, \h\lambda_{1,2}, \dotsc, \h\lambda_{1,k} \in \R$, and
$|\h\lambda_{1,j} - \lambda_{1,\tau(j)}| \leq \veps_3 \cdot \gamma_A$ for
all $j \in [k]$.

\item There exists a matrix $\h{R} \in \R^{k \times k}$ whose $j$-th column
is a right eigenvector corresponding to $\h\lambda_{1,j}$, scaled so
$\|\h{R}\e_j\|_2 = 1$ for all $j \in [k]$, such that $\|\h{R} - R_\tau\|_2
\leq \frac{\veps_4}{\|R^{-1}\|_2}$, where $R_\tau$ is the matrix obtained
by permuting the columns of $R$ with $\tau$.

\item The matrix $\h{R}$ is invertible and its inverse satisfies
$\|\h{R}^{-1} - R_\tau^{-1}\|_2 \leq \|R^{-1}\|_2 \cdot
\frac{\veps_4}{1-\veps_4}$;

\item For all $i \in \{2,3,\dotsc,k\}$ and all $j \in [k]$,
the $(j,j)$-th element of $\h{R}^{-1} \h{A}_i \h{R}$, denoted by
$\h\lambda_{i,j} := \e_j^\t \h{R}^{-1} \h{A}_i \h{R} \e_j$, satisfies
\begin{align*}
|\h\lambda_{i,j} - \lambda_{i,\tau(j)}|
& \leq
\biggl( 1 + \frac{\veps_4}{1-\veps_4} \biggr)
\cdot \biggl( 1 + \frac{\veps_4}{\sqrt{k} \cdot \kappa(R)} \biggr)
\cdot \veps_3 \cdot \gamma_A
\\
& \quad{}
+ \kappa(R)
\cdot \biggl( \frac{1}{1-\veps_4} + \frac{1}{\sqrt{k} \cdot \kappa(R)} +
\frac{1}{\sqrt{k}} \cdot \frac{\veps_4}{1-\veps_4} \biggr)
\cdot \veps_4 \cdot \lambda_{\max}
.
\end{align*}
If $\veps_4 \leq \frac12$, then $|\h\lambda_{i,j} - \lambda_{i,\tau(j)}|
\leq 3\veps_3 \cdot \gamma_A + 4\kappa(R) \cdot \veps_4 \cdot
\lambda_{\max}$.

\end{enumerate}
\end{lemma}
\begin{proof}
The first and second claims follow from applying
Lemma~\ref{lemma:eig-perturb} to $A_1$ and $\h{A}_1$.
The third claim follows from applying Lemma~\ref{lemma:inverse-perturb}
to $\h{R}$ and $R_\tau$.
To prove the last claim, first define $\v\zeta_j^\t \in \R^k$
($\h\zeta_j^\t$) to be the $j$-th row of $R_\tau^{-1}$ ($\h{R}^{-1}$), and
$\v\xi_j \in \R^k$ ($\h\xi_j$) to be the $j$-th column of $R_\tau$
($\h{R}$), so $\v\zeta_j^\t A_i \v\xi_j = \lambda_{i,\tau(j)}$ and
$\h\zeta_j^\t \h{A}_i \h\xi_j = \e_j^\t \h{R}^{-1} \h{A}_i \h{R} \e_j =
\h\lambda_{i,j}$.
By the triangle and Cauchy-Schwarz inequalities and the sub-multiplicative
property of the spectral norm,
\begin{align}
\lefteqn{
|\h\lambda_{i,j} - \lambda_{i,\tau(j)}|
} \nonumber \\
& = |\h\zeta_j^\t \h{A}_i \h\xi_j - \v\zeta_j^\t A_i \v\xi_j|
\nonumber \\
& = |\v\zeta_j^\t (\h{A}_i - A_i) \v\xi_j
+ \v\zeta_j^\t (\h{A}_i - A_i) (\h\xi_j - \v\xi_j)
+ (\h\zeta_j - \v\zeta_j)^\t (\h{A}_i - A_i) \v\xi_j
\nonumber \\
& \quad{}
+ (\h\zeta_j - \v\zeta_j)^\t (\h{A}_i - A_i) (\h\xi_j - \v\xi_j)
+ (\h\zeta_j - \v\zeta_j)^\t A_i \v\xi_j
+ \v\zeta_j^\t A_i (\h\xi_j - \v\xi_j)
+ (\h\zeta_j - \v\zeta_j)^\t A_i (\h\xi_j - \v\xi_j)|
\nonumber \\
& \leq |\v\zeta_j^\t (\h{A}_i - A_i) \v\xi_j|
+ |\v\zeta_j^\t (\h{A}_i - A_i) (\h\xi_j - \v\xi_j)|
+ |(\h\zeta_j - \v\zeta_j)^\t (\h{A}_i - A_i) \v\xi_j|
\nonumber \\
& \quad{}
+ |(\h\zeta_j - \v\zeta_j)^\t (\h{A}_i - A_i) (\h\xi_j - \v\xi_j)|
+ |(\h\zeta_j - \v\zeta_j)^\t A_i \v\xi_j|
+ |\v\zeta_j^\t A_i (\h\xi_j - \v\xi_j)|
+ |(\h\zeta_j - \v\zeta_j)^\t A_i (\h\xi_j - \v\xi_j)|
\nonumber \\
& \leq \|\v\zeta_j\|_2 \cdot \|\h{A}_i - A_i\|_2 \cdot \|\v\xi_j\|_2
+ \|\v\zeta_j\|_2 \cdot \|\h{A}_i - A_i\|_2 \cdot \|\h\xi_j - \v\xi_j\|_2
+ \|\h\zeta_j - \v\zeta_j\|_2 \cdot \|\h{A}_i - A_i\|_2 \|\v\xi_j\|_2
\nonumber \\
& \quad{}
+ \|\h\zeta_j - \v\zeta_j\|_2 \cdot \|\h{A}_i - A_i\|_2 \cdot \|\h\xi_j -
\v\xi_j\|_2
\nonumber \\
& \quad{}
+ \|\h\zeta_j - \v\zeta_j\|_2 \cdot \|\lambda_{i,\tau(j)} \v\xi_j\|_2
+ \|\lambda_{i,\tau(j)} \v\zeta_j\|_2 \cdot \|\h\xi_j - \v\xi_j\|_2
+ \|\h\zeta_j - \v\zeta_j\|_2 \cdot \|A_i\|_2 \cdot \|\h\xi_j - \v\xi_j\|_2
.
\label{eq:decomp}
\end{align}
Observe that
$\|\v\zeta_j\|_2 \leq \|R^{-1}\|_2$,
$\|\v\xi_j\|_2 \leq \|R\|_2$,
$\|\h\zeta_j - \v\zeta_j\|_2 \leq \|\h{R}^{-1} -
R_\tau^{-1}\|_2 \leq \|R^{-1}\|_2 \cdot \frac{\veps_4}{1-\veps_4}$,
$\|\h\xi_j - \v\xi_j\|_2 \leq 4k \cdot \|R^{-1}\|_2 \cdot \veps_3$ (by
Lemma~\ref{lemma:eig-perturb}), and $\|A_i\|_2 \leq \|R\|_2 \cdot (\max_j
|\lambda_{i,j}|) \cdot \|R^{-1}\|_2$.
Therefore, continuing from~\eqref{eq:decomp}, 
$|\h\lambda_{i,j} - \lambda_{i,\tau(j)}|$ is bounded as
\begin{align*}
|\h\lambda_{i,j} - \lambda_{i,\tau(j)}|
& \leq
\|R^{-1}\|_2 \cdot \|R\|_2 \cdot \eps_A
+ \|R^{-1}\|_2 \cdot \eps_A \cdot 4k \cdot \|R^{-1}\|_2 \cdot \veps_3
+ \|R^{-1}\|_2 \cdot \frac{\veps_4}{1-\veps_4} \cdot \eps_A \cdot \|R\|_2
\\
& \quad{}
+ \|R^{-1}\|_2 \cdot \frac{\veps_4}{1-\veps_4}
\cdot \eps_A \cdot 4k \cdot \|R^{-1}\|_2 \cdot \veps_3
\\
& \quad{}
+ \lambda_{\max} \cdot \|R^{-1}\|_2 \cdot \frac{\veps_4}{1-\veps_4} \cdot
\|R\|_2
+ \lambda_{\max} \cdot \|R^{-1}\|_2 \cdot 4k \cdot \|R^{-1}\|_2 \cdot
\veps_3
\\
& \quad{}
+ \|R^{-1}\|_2 \cdot \frac{\veps_4}{1-\veps_4} \cdot \|R\|_2 \cdot
\lambda_{\max} \cdot \|R^{-1}\|_2 \cdot 4k \cdot \|R^{-1}\|_2 \cdot \veps_3
\\
& = \veps_3 \cdot \gamma_A
+ \frac{\veps_4}{\sqrt{k} \cdot \kappa(R)} \cdot \veps_3 \cdot \gamma_A
+ \frac{\veps_4}{1-\veps_4} \cdot \veps_3 \cdot \gamma_A
\\
& \quad{}
+ \frac{\veps_4}{\sqrt{k} \cdot \kappa(R)} \cdot \frac{\veps_4}{1-\veps_4}
\cdot \veps_3 \cdot \gamma_A
\\
& \quad{}
+ \kappa(R) \cdot \frac{1}{1-\veps_4} \cdot \veps_4 \cdot \lambda_{\max}
+ \frac{1}{\sqrt{k}} \cdot \veps_4 \cdot \lambda_{\max}
+ \frac{\kappa(R)}{\sqrt{k}} \cdot \frac{\veps_4}{1-\veps_4} \cdot \veps_4
\cdot \lambda_{\max}
.
\end{align*}
Rearranging gives the claimed inequality.
\end{proof}

\begin{lemma} \label{lemma:normalize-eig}
Let $V \in \R^{k \times k}$ be an invertible matrix, and let $R \in \R^{k
\times k}$ be the matrix whose $j$-th column is $V\e_j / \|V\e_j\|_2$.
Then $\|R\|_2 \leq \kappa(V)$, $\|R^{-1}\|_2 \leq \kappa(V)$, and
$\kappa(R) \leq \kappa(V)^2$.
\end{lemma}
\begin{proof}
We have $R = V \diag(\|V\e_1\|_2, \|V\e_2\|_2, \dotsc, \|V\e_k\|_2)^{-1}$,
so by the sub-multiplicative property of the spectral norm, $\|R\|_2 \leq
\|V\|_2 / \min_j \|V\e_j\|_2 \leq \|V\|_2 / \sigma_k(V) = \kappa(V)$.
Similarly, $\|R^{-1}\|_2 \leq \|V^{-1}\|_2 \cdot \max_j \|V\e_j\|_2 \leq
\|V^{-1}\|_2 \cdot \|V\|_2 = \kappa(V)$.
\end{proof}

The next lemma shows that randomly projecting a collection of vectors to
$\R$ does not collapse any two too close together, nor does it send any of
them too far away from zero.

\begin{lemma} \label{lemma:eigen-gap}
Fix any $\delta \in (0,1)$ and matrix $A \in \R^{m \times n}$ (with $m \leq
n$).
Let $\v\theta \in \R^m$ be a random vector distributed uniformly over
$\sphere^{m-1}$.
\begin{enumerate}
\item $\displaystyle\Pr\biggl[ \min_{i \neq j} |\dotp{\v\theta,A(\e_i -
\e_j)}| > \frac{\min_{i \neq j} \|A(\e_i-\e_j)\|_2 \cdot
\delta}{\sqrt{em} {n \choose 2}} \biggr] \geq 1-\delta$.

\item $\displaystyle\Pr\biggl[ \forall i \in [m] , \
|\dotp{\v\theta,A\e_i}| \leq \frac{\|A\e_i\|_2}{\sqrt{m}} \Bigl(1 +
\sqrt{2\ln(m/\delta)} \Bigr) \biggr] \geq 1-\delta$.

\end{enumerate}
\end{lemma}
\begin{proof}
For the first claim, let $\delta_0 := \delta / {n \choose 2}$.
By Lemma~\ref{lemma:sanjoy}, for any fixed pair $\{i,j\} \subseteq
[n]$ and $\beta := \delta_0 / \sqrt{e}$,
\[ \Pr\biggl[ |\dotp{\v\theta,A(\e_i-\e_j)}| \leq \|A(\e_i-\e_j)\|_2 \cdot
\frac{1}{\sqrt{m}} \cdot \frac{\delta_0}{\sqrt{e}} \biggr]
\leq \exp\left(\frac12 (1 - (\delta_0^2 / e) + \ln (\delta_0^2/e)) \right)
\leq \delta_0
.
\]
Therefore the first claim follows by a union bound over all ${n \choose 2}$
pairs $\{i,j\}$.

For the second claim, apply Lemma~\ref{lemma:sanjoy} with $\beta := 1 + t$
and $t := \sqrt{2 \ln (m/\delta)}$ to obtain
\begin{align*}
\Pr\biggl[ |\dotp{\v\theta,A\e_i}| \geq \frac{\|A\e_i\|_2}{\sqrt{m}} \cdot
(1+t) \biggr]
& \leq \exp\left(\frac12 \Bigl( 1 - (1 + t)^2 + 2\ln(1 + t) \Bigr) \right)
\\
& \leq \exp\left(\frac12 \Bigl( 1 - (1 + t)^2 + 2t \Bigr) \right)
\\
& = e^{-t^2/2} = \delta / m
.
\end{align*}
Therefore the second claim follows by taking a union bound over all $i \in
[m]$.
\end{proof}

\section{Proofs and details from Section~\ref{section:applications}}
\label{appendix:applications}

In this section, we provide ommitted proofs and details from
Section~\ref{section:applications}.

\subsection{Proof of Proposition~\ref{proposition:gmm-cov}}

As in the proof of Lemma~\ref{lemma:general-moments}, it is easy to show
that
\begin{align*}
Q_{1,2,3}(\v\phi,\v\psi)
& = \E[\E[\x_1|h] \otimes \E[\x_2|h] \dotp{\v\phi,\E[\x_3 \otimes
\x_3|h]\v\psi}] \\
& = M_1 \E[ \e_h \otimes \e_h \dotp{\v\phi,(\v\mu_{3,h} \otimes \v\mu_{3,h}
+ \Sig_{3,h})\v\psi} ] M_2^\t
\\
& = M_1 \diag(\dotp{\v\phi,\v\mu_{3,t}}
\dotp{\v\psi,\v\mu_{3,t}}
+ \dotp{\v\phi,\Sig_{3,t}\v\psi}
: t \in [k]
) \diag(\v{w}) M_2^\t
.
\end{align*}
The claim then follows from the same arguments used in the proof of
Lemma~\ref{lemma:general-operator}.

\subsection{Proof of Proposition~\ref{proposition:hmm-rep}}

The conditional independence properties follow from the HMM conditional
independence assumptions.
To check the parameters, observe first that
\[
\Pr[h_1 = i |h_2 = j]
= \frac{\Pr[h_2 = j | h_1 = i] \cdot \Pr[h_1 = i]}{\Pr[h_2 = j]}
= \frac{T_{j,i} \pi_i}{(T\v\pi)_j}
= \e_i \diag(\v\pi) T^\t \diag(T\v\pi)^{-1} \e_j
\]
by Bayes' rule.
Therefore
\[ M_1\e_j = \E[\x_1|h_2 = j] = O \E[\e_{h_1} | h_2 = j] = O \diag(\v\pi)
T^\t \diag(T\v\pi)^{-1} \e_j . \]
The rest of the parameters are similar to verify.

\subsection{Learning mixtures of product distributions}
\label{appendix:product}

In this section, we show how to use \algorithmB\ with mixtures of product
distributions in $\R^n$ that satisfy an \emph{incoherence condition} on the
means $\v\mu_1, \v\mu_2, \dotsc, \v\mu_k \in \R^n$ of $k$ component
distributions.
Note that product distributions are just a special case of the more general
class of multi-view distributions, which are directly handled by
\algorithmB.

The basic idea is to randomly partition the coordinates into $\ell \geq 3$
``views'', each of roughly the same dimension.
Under the assumption that the component distributions are product
distributions, the multi-view assumption is satisfied.
What remains to be checked is that the non-degeneracy condition
(Condition~\ref{cond:general}) is satisfied.
Theorem~\ref{theorem:incoherence} (below) shows that it suffices that the
original matrix of component means have rank $k$ and satisfy the following
incoherence condition.
\begin{condition}[Incoherence condition] \label{cond:incoherence}
Let $\delta \in (0,1)$, $\ell \in [n]$, and $M = [\v\mu_1 | \v\mu_2 |
\dotsb | \v\mu_k ] \in \R^{n \times k}$ be given; let $M = USV^\t$ be the
thin singular value decomposition of $M$, where $U \in \R^{n \times k}$ is
a matrix of orthonormal columns, $S = \diag(\sigma_1(M), \sigma_2(M),
\dotsc, \sigma_k(M)) \in \R^{k \times k}$, and $V \in \R^{k \times k}$ is
orthogonal; and let
\begin{equation*}
c_M := \max_{j \in [n]} \biggl\{ \frac{n}{k} \cdot \|U^\t \e_j\|_2^2
\biggr\}
.
\end{equation*}
The following inequality holds:
\begin{equation*}
c_M \leq \frac{9}{32} \cdot \frac{n}{k\ell \ln \frac{k\ell}{\delta}}
.
\end{equation*}
\end{condition}
Note that $c_M$ is always in the interval $[1,n/k]$; it is smallest when
the left singular vectors in $U$ have $\pm 1/\sqrt{n}$ entries (as in a
Hadamard basis), and largest when the singular vectors are the coordinate
axes.
Roughly speaking, the incoherence condition requires that the
non-degeneracy of a matrix $M$ be witnessed by many vertical blocks of $M$.
When the condition is satisfied, then with high probability, a random
partitioning of the coordinates into $\ell$ groups induces a block
partitioning of $M$ into $\ell$ matrices $M_1, M_2, \dotsc, M_\ell$ (with
roughly equal number of rows) such that the $k$-th largest singular value
of $M_v$ is not much smaller than that of $M$ (for each $v \in [\ell]$).

\citet{CR08} show that under a similar condition (which they call a
\emph{spreading condition}), a random partitioning of the coordinates into
two ``views'' preserves the separation between the means of $k$ component
distributions.
They then follow this preprocessing with a projection based on the
correlations across the two views (similar to CCA).
However, their overall algorithm requires a minimum separation condition on
the means of the component distributions.
In contrast, \algorithmB\ does not require a minimum separation condition
at all in this setting.

\begin{theorem} \label{theorem:incoherence}
\def\calI{\mathcal{I}}
Assume Condition~\ref{cond:incoherence} holds.
Independently put each coordinate $i \in [n]$ into one of $\ell$ different
sets $\calI_1, \calI_2, \dotsc, \calI_\ell$ chosen uniformly at random.
With probability at least $1-\delta$, for each $v \in [\ell]$, the matrix
$M_v \in \R^{|\calI_v| \times k}$ formed by selecting the rows of $M$
indexed by $\calI_v$, satisfies
\[ \sigma_k(M_v) \geq \sigma_k(M) / (2\sqrt\ell) . \]
\end{theorem}
\begin{proof}
Follows from Lemma~\ref{lemma:partitioning} (below) together with a union
bound.
\end{proof}

\begin{lemma} \label{lemma:partitioning}
Assume Condition~\ref{cond:incoherence} holds.
Consider a random submatrix $\wh{M}$ of $M$ obtained by independently
deciding to include each row of $M$ with probability $1/\ell$.
Then
\[ \Pr\Bigl[ \sigma_k(\wh{M}) \geq \sigma_k(M) / (2\sqrt{\ell}) \Bigr] \geq
1-\delta/\ell . \]
\end{lemma}
\begin{proof}
Let $z_1, z_2, \dotsc, z_n \in \{0,1\}$ be independent indicator random
variables, each with $\Pr[z_i=1]=1/\ell$.
Note that $\wh{M}^\t \wh{M} = M^\t \diag(z_1,z_2,\dotsc,z_n) M =
\sum_{i=1}^n z_i M^\t e_ie_i^\t M$, and that
\[ \sigma_k(\wh{M})^2
= \lambda_{\min}(\wh{M}^\t \wh{M})
\geq \lambda_{\min}(S)^2 \cdot \lambda_{\min}\biggl(
\sum_{i=1}^n z_i U^\t e_ie_i^\t U \biggr) .
\]
Moreover,
$0 \preceq z_i U^\t e_i e_i^\t U \preceq (k/n) c_M I$
and $\lambda_{\min}(\E[ \sum_{i=1}^n z_i U^\t e_i e_i^\t U ]) = 1/\ell$.
By Lemma~\ref{lemma:matrix-chernoff} (a Chernoff bound on extremal
eigenvalues of random symmetric matrices),
\[ \Pr\Biggl[ \lambda_{\min}\biggl(\sum_{j=1}^d z_i U^\t e_ie_i^\t U
\biggr) \leq
\frac{1}{4\ell} \Biggr]
\leq k \cdot e^{-(3/4)^2/(2\ell c_M k/n)} \leq \delta/\ell \]
by the assumption on $c_M$.
\end{proof}

\subsection{Empirical moments for multi-view mixtures of subgaussian
distributions}
\label{appendix:subgaussian}

The required concentration behavior of the empirical moments used by
\algorithmB\ can be easily established for multi-view Gaussian mixture
models using known techniques~\citep{CKLS09}.
This is clear for the second-order statistics $\h{P}_{a,b}$ for $\{a,b\}
\in \{\{1,2\},\{1,3\}\}$, and remains true for the third-order statistics
$\h{P}_{1,2,3}$ because $\x_3$ is conditionally independent of $\x_1$ and
$\x_2$ given $h$.
The magnitude of $\dotp{\h{U}_3\v\theta_i,\x_3}$ can be bounded for all
samples (with a union bound; recall that we make the simplifying assumption
that $\h{P}_{1,3}$ is independent of $\h{P}_{1,2,3}$, and therefore so are
$\h{U}_3$ and $\h{P}_{1,2,3}$).
Therefore, one effectively only needs spectral norm error bounds for
second-order statistics, as provided by existing techniques.

Indeed, it is possible to establish Condition~\ref{cond:concentration} in
the case where the conditional distribution of $\x_v$ given $h$ (for each
view $v$) is subgaussian.
Specifically, we assume that there exists some $\alpha > 0$ such that for
each view $v$ and each component $j \in [k]$,
\begin{equation*}
\E\biggl[\exp\Bigl(\lambda \dotp{\v{u}, \cov(\x_v|h=j)^{-1/2} (\x_v -
\E[\x_v|h=j])} \Bigr) \biggr] \leq \exp(\alpha\lambda^2/2) , \quad \forall
\lambda \in \R, \v{u} \in \sphere^{d-1}
\end{equation*}
where $\cov(\x|h=j) := \E[(\x_v-\E[\x_v|h=j]) \otimes (\x_v-\E[\x_v|h=j]) |
h=j]$ is assumed to be positive definite.
Using standard techniques~(\emph{e.g.}, \citet{Vershynin12}),
Condition~\ref{cond:concentration} can be shown to hold under the above
conditions with the following parameters (for some universal constant
$c>0$):
\begin{equation*}
\begin{split}
w_{\min} & := \min_{j \in [k]} w_j \\
N_0 & := \frac{\alpha^{3/2} (d + \log(1/\delta))}{w_{\min}} \log
\frac{\alpha^{3/2} (d + \log(1/\delta))}{w_{\min}} \\
C_{a,b} & := c \cdot \Bigl( \max\Bigl\{ \|\cov(\x_v|h=j)\|_2^{1/2}, \|\E[\x_v|h=j]\|_2 : v
\in \{a,b\}, j \in [k] \Bigr\} \Bigr)^2 \\
C_{1,2,3} & := c \cdot \Bigl( \max\bigl\{ \|\cov(\x_v|h=j)\|_2^{1/2}, \|\E[\x_v|h=j]\|_2 : v
\in [3], j \in [k] \Bigr\} \Bigr)^3 \\
f(N,\delta) & := \sqrt{\frac{k^2 \log(1/\delta)}{N}}
+ \sqrt{\frac{\alpha^{3/2} \sqrt{\log(N/\delta)} (d +
\log(1/\delta))}{w_{\min} N}}
.
\end{split}
\end{equation*}

\section{General results from matrix perturbation theory}
\label{appendix:matrix}

The lemmas in this section are standard results from matrix perturbation
theory, taken from~\citet{SS90}.

\begin{lemma}[Weyl's theorem] \label{lemma:weyl}
Let $A, E\in \R^{m \times n}$ with $m \geq n$ be given.
Then
\[ \max_{i \in [n]} |\sigma_i(A+E) - \sigma_i(A)| \leq \|E\|_2
. \]
\end{lemma}
\begin{proof}
See Theorem 4.11, p.~204 in~\citet{SS90}.
\end{proof}

\begin{lemma}[Wedin's theorem] \label{lemma:wedin}
Let $A, E \in \R^{m \times n}$ with $m \geq n$ be given.
Let $A$ have the singular value decomposition
\[
\left[ \begin{array}{c} U_1^\top \\ U_2^\top \\ U_3^\top \end{array} \right]
A \left[ \begin{array}{cc} V_1 & V_2 \end{array} \right]
=
\left[ \begin{array}{cc} \Sigma_1 & 0 \\ 0 & \Sigma_2 \\ 0 & 0 \end{array}
\right]
.
\]
Let $\tl A := A + E$, with analogous singular value decomposition
$(\tl U_1, \tl U_2, \tl U_3, \tl \Sigma_1, \tl \Sigma_2, \tl V_1 \tl V_2)$.
Let $\Phi$ be the matrix of canonical angles between $\range(U_1)$ and
$\range(\tl U_1)$, and $\Theta$ be the matrix of canonical angles between
$\range(V_1)$ and $\range(\tl V_1)$.
If there exists $\delta, \alpha > 0$ such that
$\min_i \sigma_i(\tl \Sigma_1) \geq \alpha + \delta$ and
$\max_i \sigma_i(\Sigma_2) \leq \alpha$, then
$$
\max\{\|\sin \Phi\|_2, \|\sin \Theta\|_2\} \leq \frac{\|E\|_2}{\delta}.
$$
\end{lemma}
\begin{proof}
See Theorem 4.4, p.~262 in~\citet{SS90}.
\end{proof}

\begin{lemma}[Bauer-Fike theorem] \label{lemma:bauer-fike}
Let $A, E \in \R^{k \times k}$ be given.
If $A = V \diag(\lambda_1,\lambda_2,\dotsc,\lambda_k) V^{-1}$ for some
invertible $V \in \R^{k \times k}$, and $\tl{A} := A + E$ has eigenvalues
$\tl\lambda_1,\tl\lambda_2,\dotsc,\tl\lambda_k$, then
\[
\max_{i \in [k]} \min_{j \in [k]}
|\tl\lambda_i - \lambda_j| \leq \|V^{-1} E V\|_2
.
\]
\end{lemma}
\begin{proof}
See Theorem 3.3, p.~192 in~\citet{SS90}.
\end{proof}

\begin{lemma} \label{lemma:inverse-perturb}
Let $A, E \in \R^{k \times k}$ be given.
If $A$ is invertible, and $\|A^{-1} E\|_2 < 1$, then $\tl{A} := A + E$ is
invertible, and
\[
\|\tl{A}^{-1} - A^{-1}\|_2
\leq \frac{\|E\|_2 \|A^{-1}\|_2^2}{1 - \|A^{-1} E\|_2}
.
\]
\end{lemma}
\begin{proof}
See Theorem 2.5, p.~118 in~\citet{SS90}.
\end{proof}

\section{Probability inequalities}
\label{appendix:probability}

\begin{lemma}[Accuracy of empirical probabilities]
\label{lemma:discrete}
Fix $\v\mu = (\mu_1,\mu_2,\dotsc,\mu_n) \in \Delta^{m-1}$.
Let $\x$ be a random vector for which $\Pr[\x = \e_i] = \mu_i$ for all $i
\in [m]$, and let $\x_1,\x_2,\dotsc,\x_n$ be $n$ independent copies of
$\x$.
Set $\h\mu := (1/n)\sum_{i=1}^n \x_i$.
For all $t > 0$,
\[ \Pr\biggl[ \|\h\mu - \v\mu\|_2 > \frac{1 + \sqrt{t}}{\sqrt{n}}
\biggr] \leq e^{-t} . \]
\end{lemma}
\begin{proof}
This is a standard application of McDiarmid's inequality (using the fact
that $\|\h\mu - \v\mu\|_2$ has $\sqrt{2} / n$ bounded differences when a
single $\x_i$ is changed), together with the bound $\E[\|\h\mu - \v\mu\|_2]
\leq 1/\sqrt{n}$.
See Proposition 19 in~\cite{HKZ12}.
\end{proof}

\begin{lemma}[Random projection]
\label{lemma:sanjoy}
Let $\v\theta \in \R^n$ be a random vector distributed uniformly over
$\sphere^{n-1}$, and fix a vector $\v{v} \in \R^n$.
\begin{enumerate}
\item If $\beta \in (0,1)$, then
\[
\Pr\biggl[ |\dotp{\v\theta,\v{v}}| \leq \|\v{v}\|_2 \cdot \frac{1}{\sqrt{n}}
\cdot \beta \biggr] \leq \exp\biggl(\frac12(1 - \beta^2 + \ln
\beta^2)\biggr)
.
\]

\item If $\beta > 1$, then
\[
\Pr\biggl[ |\dotp{\v\theta,\v{v}}| \geq \|\v{v}\|_2 \cdot \frac{1}{\sqrt{n}}
\cdot \beta \biggr] \leq \exp\biggl(\frac12(1 - \beta^2 + \ln
\beta^2)\biggr)
.
\]

\end{enumerate}
\end{lemma}
\begin{proof}
This is a special case of Lemma 2.2 from~\citet{DG03}.
\end{proof}

\begin{lemma}[Matrix Chernoff bound] \label{lemma:matrix-chernoff}
Let $X_1, X_2, \dotsc, X_n$ be independent and symmetric $m \times m$
random matrices such that $0 \preceq X_i \preceq r I$, and set $l :=
\lambda_{\min}(\E[X_1 + X_2 + \dotsb + X_n])$.
For any $\eps \in [0,1]$,
\[ \Pr\Biggl[ \lambda_{\min}\biggl(\sum_{i=1}^n X_i\biggr) \leq
(1-\eps) \cdot l \Biggr] \leq m \cdot e^{-\eps^2 l / (2r)} . \]
\end{lemma}
\begin{proof}
This is a direct corollary of Theorem 19 from~\citet{AhlWin02}.
\end{proof}

\section{Insufficiency of second-order moments}
\label{appendix:nonident}

\citet{Chang96} shows that a simple class of Markov models used in
mathematical phylogenetics cannot be identified from pair-wise
probabilities alone.
Below, we restate (a specialization of) this result in terms of the
document topic model from Section~\ref{section:topic-setting}.
\begin{proposition}[\citealp{Chang96}]
\label{proposition:nonident}
Consider the model from Section~\ref{section:topic-setting} on
$(h,x_1,x_2,\dotsc,x_\ell)$ with parameters $M$ and $\v{w}$.
Let $Q \in \R^{k \times k}$ be an invertible matrix such that the following
hold:
\begin{enumerate}
\item $\v1^\t Q = \v1^\t$;

\item $MQ^{-1}$, $Q\diag(\v{w})M^\t\diag(M\v{w})^{-1}$, and $Q\v{w}$ have
non-negative entries;

\item $Q \diag(\v{w}) Q^\t$ is a diagonal matrix.

\end{enumerate}
Then the marginal distribution over $(x_1,x_2)$ is identical to that in the
case where the model has parameters $\tl{M} := MQ^{-1}$ and $\tl{w} :=
Q\v{w}$.
\end{proposition}
A simple example for $d = k = 2$ can be obtained from
\begin{align*}
M & := \begin{bmatrix} p & 1-p \\ 1-p & p \end{bmatrix} ,
& \v{w} & := \begin{bmatrix} 1/2 \\ 1/2 \end{bmatrix} ,
& Q & := \begin{bmatrix} p & \frac{1 + \sqrt{1+4p(1-p)}}{2} \\ 1-p &
\frac{1 - \sqrt{1+4p(1-p)}}{2} \end{bmatrix}
\end{align*}
for some $p \in (0,1)$.
We take $p = 0.25$, in which case $Q$ satisfies the conditions of
Proposition~\ref{proposition:nonident}, and
\begin{align*}
M & = \begin{bmatrix} 0.25 & 0.75 \\ 0.75 & 0.25 \end{bmatrix} ,
& \v{w} & = \begin{bmatrix} 0.5 \\ 0.5 \end{bmatrix} , \\
\tl{M} = MQ^{-1}
& \approx \begin{bmatrix} 0.6614 & 0.1129 \\ 0.3386 & 0.8871 \end{bmatrix} ,
& \tl{w} = Q\v{w}
& \approx \begin{bmatrix} 0.7057 \\ 0.2943 \end{bmatrix}
.
\end{align*}
In this case, both $(M,\v{w})$ and $(\tl{M},\tl{w})$ give rise to the same
pair-wise probabilities
\begin{equation*}
M\diag(\v{w})M^\t
= \tl{M}\diag(\tl{w})\tl{M}^\t
\approx \begin{bmatrix}
0.3125 & 0.1875 \\
0.1875 & 0.3125
\end{bmatrix}
.
\end{equation*}
However, the triple-wise probabilities, for $\eta = (1,0)$, differ:
for $(M,\v{w})$, we have
\begin{equation*}
M\diag(M^\t\eta)\diag(\v{w})M^\t
\approx \begin{bmatrix}
0.2188 & 0.0938 \\
0.0938 & 0.0938
\end{bmatrix}
;
\end{equation*}
while for $(\tl{M},\tl{w})$, we have 
\begin{equation*}
\tl{M}\diag(\tl{M}^\t\eta)\diag(\tl{w})\tl{M}^\t
\approx \begin{bmatrix}
0.2046 & 0.1079 \\
0.1079 & 0.0796
\end{bmatrix}
.
\end{equation*}

\end{document}